\documentclass[acmsmall,screen]{acmart}
\AtBeginDocument{}

\setcopyright{acmlicensed}
\copyrightyear{2025}
\acmYear{2025}
\acmDOI{10.48550/arXiv.2509.08578}

\acmJournal{TIST}
\acmVolume{37}
\acmNumber{4}
\acmArticle{111}
\acmMonth{8}

\begin{document}

\title{Multi-modal Adaptive Estimation for Temporal Respiratory Disease Outbreak}

\author{Hong Liu}
\authornote{Both authors contributed equally to this research.}
\email{2240009518@student.must.edu.mo}
\orcid{0009-0008-3411-5304}

\author{Kerui Cen}
\authornotemark[1]
\email{corycen1@outlook.com}
\affiliation{%
  \institution{Respiratory Disease AI Laboratory in Epidemic Intelligence and Applications of Medical Big Data Instruments, Macau University of Science and Technology}
  \city{Taipa}
  \country{Macau SAR, China}}
\affiliation{%
  \institution{Faculty of Innovation Engineering, Macau University of Science and Technology}
  \city{Taipa}
  \country{Macau SAR, China}}
\affiliation{%
  \institution{Institute of Systems Engineering, Macau University of Science and Technology}
  \city{Taipa}
  \country{Macau SAR, China}}

\author{Yanxing Chen}
\email{dchen8968@gmail.com}
\affiliation{%
  \institution{Respiratory Disease AI Laboratory in Epidemic Intelligence and Applications of Medical Big Data Instruments, Macau University of Science and Technology}
  \city{Taipa}
  \country{Macau SAR, China}}
\affiliation{%
  \institution{School of Business, Macau University of Science and Technology}
  \city{Taipa}
  \country{Macau SAR, China}}

\author{Zige Liu}
\email{zgliu@must.edu.mo}
\affiliation{%
  \institution{Respiratory Disease AI Laboratory in Epidemic Intelligence and Applications of Medical Big Data Instruments, Macau University of Science and Technology}
  \city{Taipa}
  \country{Macau SAR, China}}
\affiliation{%
  \institution{Faculty of Innovation Engineering, Macau University of Science and Technology}
  \city{Taipa}
  \country{Macau SAR, China}}
\affiliation{%
  \institution{Institute of Systems Engineering, Macau University of Science and Technology}
  \city{Taipa}
  \country{Macau SAR, China}}

\author{Dong Chen}
\email{dchen@must.edu.mo}
\affiliation{%
  \institution{Respiratory Disease AI Laboratory in Epidemic Intelligence and Applications of Medical Big Data Instruments, Macau University of Science and Technology}
  \city{Taipa}
  \country{Macau SAR, China}}

\author{Zifeng Yang}
\authornote{Corresponding author.}
\email{jeffyah@163.com}
\affiliation{%
  \institution{State Key Laboratory of Respiratory Disease, National Clinical Research Center for Respiratory Disease, Guangzhou Institute of Respiratory Health, The First Affiliated Hospital of Guangzhou Medical University}
  \city{Guangzhou}
  \country{China}}
\affiliation{%
  \institution{Guangzhou National Laboratory}
  \city{Guangzhou}
  \country{China}}

\author{Chitin Hon}
\authornotemark[2]
\email{cthon@must.edu.mo}
\affiliation{%
  \institution{Respiratory Disease AI Laboratory in Epidemic Intelligence and Applications of Medical Big Data Instruments, Macau University of Science and Technology}
  \city{Taipa}
  \country{Macau SAR, China}}
\affiliation{%
  \institution{Faculty of Innovation Engineering, Macau University of Science and Technology}
  \city{Taipa}
  \country{Macau SAR, China}}
\affiliation{%
  \institution{Institute of Systems Engineering, Macau University of Science and Technology}
  \city{Taipa}
  \country{Macau SAR, China}}
%
\renewcommand{\shortauthors}{Liu, Cen et al}

\begin{abstract}
Timely and robust influenza incidence forecasting is critical for public health decision-making. This paper presents MAESTRO (Multi-modal Adaptive Estimation for Temporal Respiratory Disease Outbreak), a novel, unified framework that synergistically integrates advanced spectro-temporal modeling with multi-modal data fusion, including surveillance, web search trends, and meteorological data. By adaptively weighting heterogeneous data sources and decomposing complex time series patterns, the model achieves robust and accurate forecasts. Evaluated on over 11 years of Hong Kong influenza data (excluding the COVID-19 period), MAESTRO demonstrates state-of-the-art performance, achieving a superior model fit with an R² of 0.956. Extensive ablations confirm the significant contributions of its multi-modal and spectro-temporal components. The modular and reproducible pipeline is made publicly available to facilitate deployment and extension to other regions and pathogens, presenting a powerful tool for epidemiological forecasting.
\end{abstract}



\received{15 September 2025}
\received[revised]{20 September 2025}
\received[accepted]{30 September 2025}

\maketitle
\section{Introduction}
Seasonal influenza exerts a substantial burden on healthcare systems and society. Accurate short-horizon forecasting of influenza incidence is crucial for proactive resource allocation, timely public health alerts, and effective communication strategies. However, achieving such forecasts is challenging. Epidemiological data are often a complex blend of signals: traditional surveillance data can be noisy and lagged, while exogenous drivers like web search trends and meteorological conditions introduce distinct temporal dynamics and periodicities. This inherent heterogeneity demands a new class of models capable of not just processing multi-modal data, but also discerning and adapting to the most salient patterns across both time and frequency domains.

To address these challenges, this paper proposes MAESTRO (Multi-modal Adaptive Estimation for Temporal Respiratory Disease Outbreak). Rather than relying on a single monolithic architecture, MAESTRO is founded on a principled design that synergistically integrates specialized modules to tackle distinct aspects of the forecasting problem. It systematically decomposes time series to separate long-term trends from seasonal fluctuations and then processes these components through parallel pathways that include advanced state-space models for capturing long-range dependencies and frequency-domain analysis for identifying robust periodic patterns. This spectro-temporal approach allows the model to build a rich, multi-faceted understanding of the underlying dynamics. The insights from these diverse analytical perspectives are then dynamically synthesized through an adaptive ensemble mechanism, which intelligently weights each component's contribution based on the evolving data context.

To validate the approach, MAESTRO was rigorously evaluated on a comprehensive real-world dataset from Hong Kong, spanning over a decade of surveillance, Google Trends, and meteorological data (excluding the COVID-19 pandemic period from training to ensure data integrity). Moreover, extensive ablation studies were conducted to systematically verify the contributions of each architectural component and data modality, providing a clear understanding of the model's inner workings.

This work makes three primary contributions. First, a novel, unified spectro-temporal architecture is introduced that synergistically couples series decomposition, state-space modeling, and frequency-domain learning to specifically address the complexities of epidemiological forecasting. Second, an effective multi-modal integration strategy is proposed that cohesively merges heterogeneous data sources—epidemiological surveillance, web search trends, and weather—through cross-channel attention and an adaptive fusion framework. Third, through a comprehensive empirical study on influenza in Hong Kong, the model's state-of-the-art accuracy and robustness are demonstrated, and full reproducibility is ensured by making the entire pipeline publicly available, thereby facilitating its adoption and extension to other diseases and regions.

\section{Related Work}
Time Series Forecasting. Deep learning has revolutionized multivariate time series forecasting. Early approaches based on recurrent neural networks (e.g., LSTM/GRU) excelled at capturing sequential dependencies but faced challenges with long-range modeling. The advent of the Transformer architecture\cite{vaswani2017attention} and its numerous variants (e.g., Informer\cite{zhou2021informer}, Autoformer\cite{wu2021autoformer}) significantly improved performance by capturing long-distance dependencies, though often at the cost of quadratic complexity. To address efficiency and specific temporal patterns, recent research has explored new paradigms. Architectures like TimesNet\cite{wu2023timesnet} leverage spectral analysis and convolutional structures to model periodicity, while others like DLinear\cite{zeng2023dlinear} have demonstrated the surprising effectiveness of simple linear heads. Most recently, state-space models (SSMs) such as S4\cite{gu2022s4} and Mamba\cite{gu2023mamba} have emerged as a powerful alternative, offering linear-time complexity for long-context modeling. However, these models often focus on a single architectural principle, lacking a synergistic integration of spectro-temporal features.

Building upon these advances in general time series forecasting, the field of epidemiological and influenza forecasting has developed its own specialized approaches. Beyond general time series methods, epidemiological forecasting has a rich history. Classical statistical models like the Box-Jenkins methodology\cite{box2015time} and mechanistic models have been widely used. Machine learning approaches have gained prominence by incorporating external data signals that are highly correlated with disease transmission. Seminal works include using web search data (e.g., Google Flu Trends\cite{ginsberg2009detecting}), environmental factors like absolute humidity\cite{shaman2010absolute}, and human mobility data\cite{senanayake2016predicting}. Deep learning models, such as DeepFlu\cite{wu2018deepflu}, have further advanced the field by learning complex, non-linear relationships from these multi-modal inputs.

The success of epidemiological forecasting heavily depends on effectively integrating heterogeneous data sources. Attention-based fusion and adaptive weighting mechanisms have proven effective for this task. While these methods are powerful, few studies have attempted to create a unified framework that simultaneously leverages advanced temporal modeling (like SSMs), frequency-domain analysis, and adaptive multi-modal fusion specifically for the challenges of epidemiological forecasting. 

MAESTRO addresses this gap. It proposes a novel, modular architecture that synergistically combines time series decomposition, state-space modeling (Mamba), and frequency-domain learning. Crucially, it employs a cross-channel attention mechanism to robustly fuse multi-modal signals, positioning it as a comprehensive and powerful tool for modern epidemiological forecasting.

The principles of Multi-modal fusion and attention mechanisms, central to MAESTRO, have demonstrated significant value across a diverse range of applications beyond epidemiological forecasting. For instance, in industrial settings, methods combining LiDAR and camera data through multimodal fusion have been developed for precise defect detection \cite{LiTIM2025AFD}. Similarly, attention mechanisms have been effectively utilized in complex detection tasks, such as identifying bird collision risks \cite{LiERA2022Bird} and enhancing object detection models for industrial quality control \cite{LiJSEN2024YOLO}. Furthermore, the fusion of advanced signal processing techniques like Adaptive Fourier Decomposition (AFD) with deep learning has shown promise in enhancing the interpretation of biomedical signals, such as spirometry data \cite{LiTIM2025Spiro}. These examples underscore the broad applicability and robustness of the core techniques that MAESTRO integrates and adapts for the specific challenges of forecasting respiratory diseases.

\section{Method}
\subsection{Problem Formulation}
Given a multivariate time series $\mathbf{X}\in\mathbb{R}^{T\times D}$ (surveillance, trends, weather) and target $y\in\mathbb{R}$ (influenza positive rate), the goal is to predict $y_{t+1: t+H}$ from $\mathbf{X}_{t-L+1:t}$ with horizon $H$ and window $L$.

\subsection{Notation and Assumptions}
The elementwise (Hadamard) product is denoted by $\odot$ and convolution by $\ast$; $\mathcal{F}(\cdot)$ and $\mathcal{F}^{-1}(\cdot)$ are the 1D FFT and inverse FFT along time. Modalities are indexed by $m\in\{1,\dots,M\}$, channels by $i\in\{1,\dots,C\}$, and time by $t\in\{1,\dots,T\}$. Hidden width is $d$. The moving-average operator $\mathcal{D}_K$ is a length-$K$ low-pass filter (trend), and $I-\mathcal{D}_K$ yields the high-pass residual (seasonal). $\operatorname{softmax}$ is used for normalized exponentiation over the specified axis, $\operatorname{softplus}(x)=\log(1+e^x)$ for nonnegativity, $\operatorname{pool}$ for temporal pooling (e.g., mean/max), and $\operatorname{MLP}$ for a small feed-forward network. Unless stated, $\|\cdot\|$ denotes the Euclidean/operator norm and $\rho(\cdot)$ the spectral radius.

\subsection{Architecture Overview}
The MAESTRO model processes input data through a multi-stage pipeline designed to capture temporal and spectral features robustly. The internal data flow is as follows:
\begin{enumerate}
    \item \textbf{Series Decomposition:} If the use\_decomp flag is enabled, the input data is first decomposed into trend and seasonal components. This is achieved using an additive decomposition method based on a moving average filter, as implemented in the SeriesDecomposition module. Specifically, the trend component is extracted by applying a 1D average pooling layer with a specified kernel size. To handle boundary effects, the input sequence is padded by repeating the first and last values. The seasonal component is then obtained by subtracting the calculated trend from the original series.
    \item \textbf{Spectro-Temporal Block:} The decomposed components (or the raw input data if decomposition is disabled) are then processed through a unified spectro-temporal pipeline:
    \begin{itemize}
        \item Normalization: The time series data is normalized by subtracting its mean and dividing by its standard deviation. These statistics are saved for denormalizing the final output.
        \item Embedding: The normalized components are fed into a DataEmbedding layer. This layer uses a 1D convolutional network (TokenEmbedding) to project the input features to the model's hidden dimension (d\_model) and adds sinusoidal positional encodings (PositionalEmbedding) to retain temporal information.
        \item Component-wise Encoding: If decomposition is used, the seasonal and trend embeddings are processed by two separate Encoder stacks (seasonal\_encoder and trend\_encoder). Otherwise, a single Encoder processes the unified data. Each encoder consists of multiple EncoderLayer modules, which apply multi-head self-attention followed by a position-wise feed-forward network.
        \item MambaBlock: A Mamba state-space model captures long-range dependencies efficiently.
        \item MultiScaleTemporalConv: Parallel 1D convolutions with different kernel sizes capture local patterns at various temporal scales, and their outputs are fused using an adaptive weighting mechanism.
        \item FrequencyDomainModule: The data is transformed into the frequency domain using FFT. A learnable filter is applied, and the result is transformed back to the time domain via an inverse FFT to capture periodic patterns.
        \item BiLSTMModule: A bidirectional LSTM processes the sequence to capture temporal dependencies in both forward and backward directions.
        \item Adaptive Weighting Fusion: If decomposition is used, the enhanced seasonal and trend representations are combined using an AdaptiveWeighting module, which learns to dynamically weigh the importance of each component.
    \end{itemize}
    \item \textbf{Cross-Channel Attention:} If the use\_cross\_attn flag is enabled, the representation is processed by a CrossChannelAttention module. This module applies attention across different input variables (channels) to explicitly model their inter-dependencies.
    \item \textbf{Temporal Projection:} The resulting representation is projected from the input sequence length to the desired prediction length using a linear layer in the TemporalProjection module.
    \item \textbf{Uncertainty Estimation:} The final prediction is generated through one of two paths:
    \begin{itemize}
        \item If estimate\_uncertainty is enabled, an UncertaintyEstimator module produces both a mean prediction and a standard deviation for uncertainty quantification.
        \item Otherwise, a single linear layer (output\_proj) maps the features to the final output.
    \end{itemize}
      \item \textbf{Denormalization:} The output is denormalized using the mean and standard deviation computed in the initial normalization step to produce the final forecast.
\end{enumerate}
This modular and hierarchical architecture allows MAESTRO to effectively integrate diverse data modalities and capture a wide range of spectro-temporal dynamics. In summary, a moving-average block decomposes inputs into seasonal and trend components. These streams are then embedded and encoded with Transformer encoders. Residual Mamba state-space, multi-scale temporal convolutions, and a frequency-domain module (learnable filters in FFT space) are added. Cross-channel attention captures inter-variable dependencies. Finally, a linear temporal projection maps encoder outputs to the prediction horizon $H$, and an optional uncertainty head outputs standard deviations via the Softplus activation function.

\begin{figure*}[htbp]
  \centering
  \includegraphics[width=0.9\textwidth]{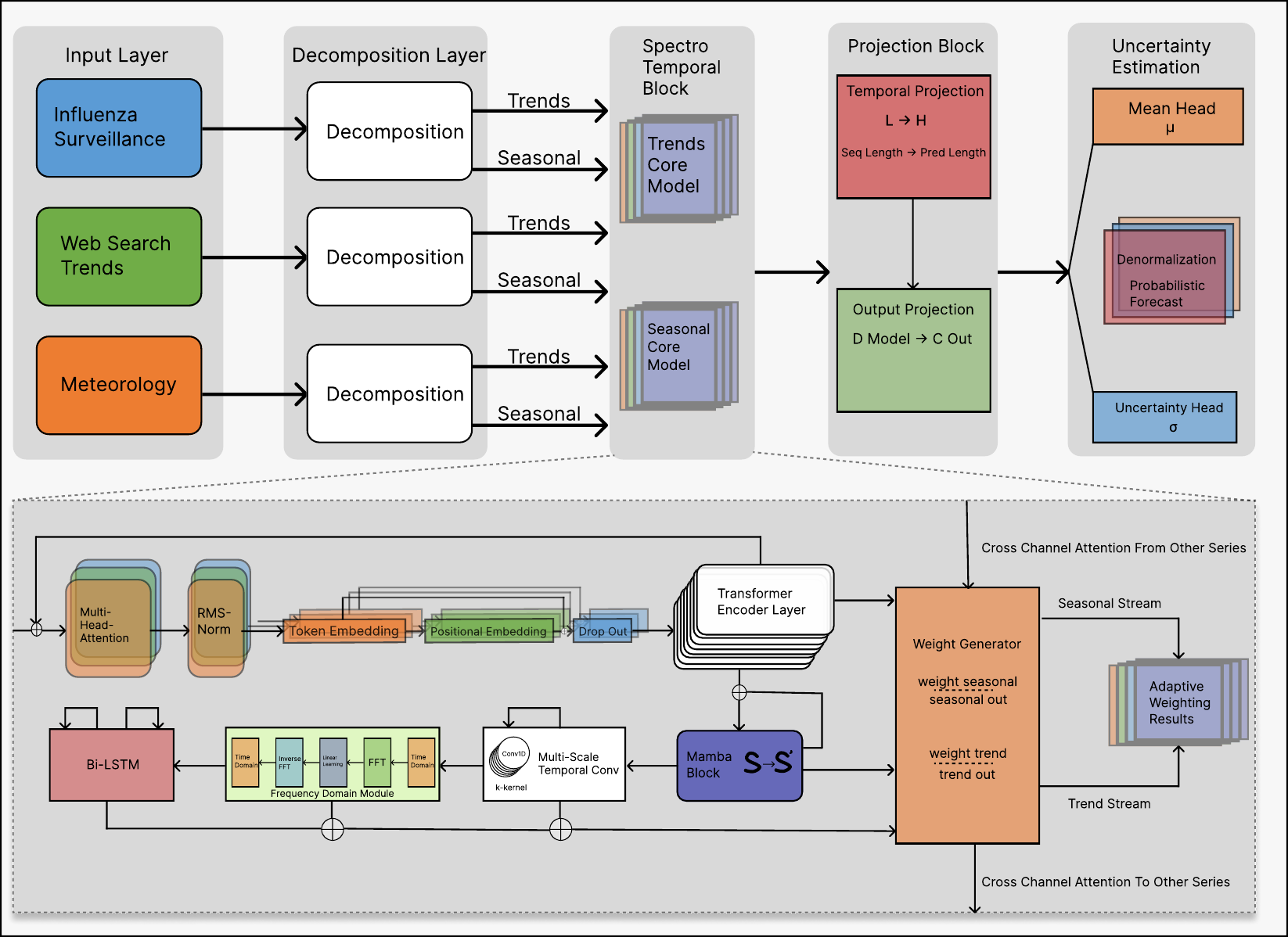}
  \caption{The MAESTRO architecture, illustrating the detailed data flow. 1) The input series is optionally decomposed into trend and seasonal components, and then each component is normalized. 2) Each component is embedded using token and positional embeddings. 3) Parallel encoders process the seasonal and trend data, followed by a series of residual spectro-temporal enhancement blocks (Mamba, Multi-Scale Temporal Conv, Frequency-Domain Module, and BiLSTM). 4) The enhanced representations are adaptively fused. 5) A temporal projection layer maps the sequence to the desired output horizon. 6) The final output is generated and denormalized to produce the forecast. The diagram shows the flow for a single modality; this process is repeated for all input modalities before a final cross-channel fusion.}
  \label{fig:architecture}
\end{figure*}

\subsection{Mamba State-Space Model}
To efficiently capture long-range dependencies in the sequence, a Mamba state-space model\cite{gu2023mamba} is employed. An SSM is fundamentally a continuous-time linear system that models how a latent state $h(t) \in \mathbb{R}^N$ evolves based on an input signal $u(t)$, and how this state produces an output $y(t)$\cite{gu2022s4}. The system is defined by:
\begin{equation}
\frac{\mathrm{d}h(t)}{\mathrm{d}t} = A h(t) + B u(t),\quad y(t) = C h(t) + D u(t),
\end{equation}
where $A \in \mathbb{R}^{N \times N}$ is the state matrix governing the internal dynamics, while $B \in \mathbb{R}^{N \times d}$, $C \in \mathbb{R}^{d \times N}$, and $D \in \mathbb{R}^{d \times d}$ are matrices that project the input, state, and input to the output, respectively.

For use in a neural network, this continuous system is discretized. A key innovation of Mamba is making this discretization process dependent on the input data itself, allowing the model to selectively focus on or ignore information\cite{gu2023mamba}. This is achieved by introducing an input-dependent step size $\Delta_t$. Using a zero-order hold, which assumes the input is held constant across the sampling interval, the continuous-time state-space matrices ($A, B$) are transformed into their discrete-time counterparts ($A_t^{\mathrm{d}}, B_t^{\mathrm{d}}$) for each time step $t$:
\begin{equation}
A_t^{\mathrm{d}} = \exp(A\Delta_t),\qquad B_t^{\mathrm{d}} = \left(\int_0^{\Delta_t} \exp(A\tau)\,\mathrm{d}\tau\right)B.
\end{equation}
Here, the matrix exponential $\exp(A\Delta_t)$ describes the evolution of the system's state over the discrete time step $\Delta_t$. For the input matrix discretization, $\tau$ serves as an integration variable for time within the interval $[0, \Delta_t]$. The integral term $\int_0^{\Delta_t} \exp(A\tau)\,\mathrm{d}\tau$ effectively accumulates the state transition's impact on the input over that interval. The crucial feature is that $\Delta_t$ is not fixed but is dynamically computed from the current input $x_t$, making the model's temporal resolution adaptive.

The full recurrence relation for an input sequence $x_t \in \mathbb{R}^d$ is defined as follows. First, the input $x_t$ is used to derive three key components: an adaptive step size $\Delta_t$, a projected input $v_t$, and a gating signal $s_t$.
\begin{equation}
\Delta_t = \operatorname{softplus}(w_\Delta^\top x_t + b_\Delta), \quad v_t = W_v x_t, \quad s_t = \sigma(W_s x_t + b_s).
\end{equation}
The step size $\Delta_t$ is guaranteed to be positive by the $\operatorname{softplus}$ activation. The vectors $w_\Delta, b_\Delta$ and matrices $W_v, W_s$ are learnable parameters. These components are then used to update the model's hidden state $h_t$ and a secondary gated state $z_t$:
\begin{equation}
 h_t = A_t^{\mathrm{d}} h_{t-1} + B_t^{\mathrm{d}} v_t,\quad z_t = s_t \odot (C h_t) + (1-s_t) \odot z_{t-1}.
\end{equation}
The primary hidden state $h_t$ evolves according to the discretized state-space dynamics. The secondary state $z_t$ is updated through a gating mechanism, where the gate $s_t$ (a value between 0 and 1 from the sigmoid function $\sigma$) determines the blend between the current SSM output, $C h_t$ (where $C$ is a learnable output matrix), and the previous state $z_{t-1}$. This structure, using element-wise multiplication $\odot$, allows the model to selectively retain or overwrite information over time.

To ensure stability, the continuous state matrix $A$ is parameterized as $A=-\operatorname{softplus}(\widehat{A})$, where $\widehat{A}$ is a learnable matrix (typically diagonal or low-rank-plus-diagonal). This forces the eigenvalues of $A$ to have negative real parts, which in turn guarantees that the spectral radius $\rho$ of the discretized state matrix, $\rho\big(\exp(A\Delta_t)\big)$, remains less than 1. This condition is necessary to prevent the hidden state $h_t$ from growing without bound. Despite its recurrent formulation, the model can be implemented with a highly parallel scan operation, achieving a computational complexity of $\mathcal{O}(T d)$\cite{gu2023mamba}.

\subsection{Multi-Scale Temporal Convolutions}
To capture local temporal patterns and contextual information at various resolutions, a multi-scale temporal convolution module is incorporated. This module operates on the hidden representations $H^m \in \mathbb{R}^{T \times d}$ from the preceding block. It employs a set of one-dimensional convolutional layers (Conv1D) with varying kernel sizes $k \in \mathcal{K}$ and dilation rates $\delta \in \mathcal{D}$. The kernel size $k$ determines the width of the local temporal window, while the dilation rate $\delta$ allows the receptive field to expand exponentially, enabling the capture of dependencies over longer ranges without a corresponding increase in computational cost.

For efficiency, depthwise separable convolutions are utilized, which factorize a standard convolution into a depthwise and a pointwise convolution\cite{wu2023timesnet}. For each combination of kernel size $k$ and dilation $\delta$, a convolutional layer processes the input sequence:
\begin{equation}
U_{k,\delta}^m = \mathrm{Conv1D}_{k,\delta}(H^m) \in \mathbb{R}^{T \times d}.
\end{equation}
Here, $U_{k,\delta}^m$ represents the feature map containing temporal patterns extracted at a specific scale defined by the pair $(k, \delta)$.

To dynamically integrate information from these different scales, the resulting feature maps are aggregated using an attention-style weighting mechanism\cite{vaswani2017attention}. This allows the model to adaptively emphasize the most relevant temporal patterns for the given input. The aggregated representation $U^m$ is computed as a weighted sum:
\begin{equation}
U^m = \sum_{k \in \mathcal{K}, \delta \in \mathcal{D}} \alpha_{k,\delta}^m \, U_{k,\delta}^m.
\end{equation}
The attention weights $\alpha_{k,\delta}^m$ are dynamically generated based on the input representation $H^m$. Specifically, a small trainable query function $q$ projects $H^m$ into a score for each scale, and a softmax function normalizes these scores into a valid probability distribution:
\begin{equation}
\alpha_{k,\delta}^m = \operatorname{softmax}_{k,\delta}\big(q(H^m)\big).
\end{equation}
This adaptive aggregation of features from multiple temporal scales allows the model to capture complex local patterns, which complements the global periodic information extracted by the frequency-domain module.

\subsection{Frequency-Domain Module}
To capture global periodicities and trends, a frequency-domain module is employed that operates on the hidden sequence $H^m \in \mathbb{R}^{T \times d}$ from the previous block. This approach is inspired by classical signal processing and recent advances in time series forecasting that leverage frequency analysis\cite{zhou2022fedformer, wu2021autoformer}.

First, the sequence is transformed into the frequency domain using a one-dimensional Fast Fourier Transform (FFT) along the time axis:
\begin{equation}
Z^m = \mathcal{F}(H^m) \in \mathbb{C}^{T \times d},
\end{equation}
where $\mathcal{F}$ denotes the FFT operator, and $Z^m$ is the complex-valued spectral representation of the input sequence.

Next, spectral filtering is performed by applying a learnable complex-valued spectral mask, $W_f^m \in \mathbb{C}^{T \times d}$, to the representation $Z^m$. This mask is parameterized either by its real and imaginary parts or by its magnitude and phase. The filtering operation is an element-wise (Hadamard) product:
\begin{equation}
\widetilde{Z}^m = W_f^m \odot Z^m.
\end{equation}
Here, $\widetilde{Z}^m$ denotes the filtered spectral representation, where the tilde ($\sim$) explicitly indicates that $Z^m$ has been transformed or modified by the spectral filtering operation. This notation distinguishes the processed spectral data from its original form, $Z^m$, which is the direct FFT output of $H^m$. This step allows the model to selectively amplify or attenuate specific frequencies, effectively learning a frequency-domain filter to extract periodic patterns and suppress noise. The filtered spectral representation $\widetilde{Z}^m$ is then transformed back to the time domain using the inverse FFT ($\mathcal{F}^{-1}$):
\begin{equation}
\widetilde{H}^m = \mathcal{F}^{-1}(\widetilde{Z}^m) \in \mathbb{R}^{T \times d}.
\end{equation}

To enhance robustness and capture patterns across different time scales, a multi-window analysis can be optionally employed. Let $\mathcal{S}$ be a set of different analysis configurations (e.g., using different FFT window functions or segment lengths). For each configuration $s \in \mathcal{S}$, a filtered time-domain representation $\widetilde{H}_s^m$ is computed. These representations are then aggregated using an adaptive weighting scheme:
\begin{equation}
\widetilde{H}^m = \sum_{s \in \mathcal{S}} \pi_s^m \, \widetilde{H}_s^m.
\end{equation}
The weights $\pi_s^m$ are generated dynamically, allowing the model to emphasize the most informative spectral view for a given input. They are computed via a softmax over scores produced by a small multi-layer perceptron (MLP), $g_s$, which takes a summarized representation of the input as its argument:
\begin{equation}
\pi_s^m = \operatorname{softmax}_s\big( g_s(\operatorname{pool}(H^m)) \big),
\end{equation}
where $\operatorname{pool}$ is a temporal pooling operator (e.g., average pooling) that creates a fixed-size summary statistic of the sequence $H^m$.

Finally, to ensure the stability and interpretability of the learned filter, the spectral mask $W_f^m$ is regularized. The spectral gain is constrained by enforcing $\sup_\omega \lVert W_f^m(\omega)\rVert \le G_f$, where $G_f$ is a predefined upper bound, implemented via parameterization and clipping. Additionally, penalties may be applied to encourage smoothness (e.g., penalizing the first-difference of the filter magnitudes) and sparsity (e.g., an L1 penalty on the magnitudes) to promote selective and stable spectral shaping.

\subsection{Cross-Channel Attention}
To explicitly model the interactions and dependencies between the $C$ different variables (channels) of the multivariate time series, a cross-channel attention module is incorporated. This allows the model to dynamically route information between related variables. The input to this module is a tensor $H \in \mathbb{R}^{T \times C \times d}$, representing the features for $C$ channels over $T$ time steps.

Following the standard self-attention mechanism\cite{vaswani2017attention}, query, key, and value representations are first generated. However, to capture the relationships between channels rather than time steps, attention scores are computed based on a summary of each channel's representation across time. Specifically, a temporal pooling operator, $\operatorname{pool}_t$, is first applied to the time series of each channel $H_{:,i,:} \in \mathbb{R}^{T \times d}$ to obtain a channel-specific summary vector $h_i^{\text{pool}} \in \mathbb{R}^d$. These summary vectors are then linearly projected to form the query ($q_i$) and key ($k_j$) vectors:
\begin{equation}
q_i = \operatorname{pool}_t(H_{:,i,:}) W_Q, \quad k_j = \operatorname{pool}_t(H_{:,j,:}) W_K,
\end{equation}
where $W_Q \in \mathbb{R}^{d \times d_k}$ and $W_K \in \mathbb{R}^{d \times d_k}$ are learnable projection matrices. The attention weight $\alpha_{i,j}$, which quantifies the influence of channel $j$ on channel $i$, is then computed using the scaled dot-product:
\begin{equation}
\alpha_{i,j} = \operatorname{softmax}_j\left( \frac{q_i k_j^\top}{\sqrt{d_k}} \right).
\end{equation}
This results in an attention matrix $\alpha \in \mathbb{R}^{C \times C}$ that is shared across all time steps.

Finally, the output representation $\widetilde{H}_{t,i,:}$ for channel $i$ at time step $t$ is computed by mixing the information from all channels, weighted by the attention scores. For each channel $j$ and time step $t$, a value vector $v_{t,j}$ is created using a projection matrix $W_V \in \mathbb{R}^{d \times d_v}$. The final output is the weighted sum of these value vectors:
\begin{equation}
\widetilde{H}_{t,i,:} = \sum_{j=1}^C \alpha_{i,j} (H_{t,j,:} W_V).
\end{equation}
This mechanism captures inter-variable couplings with a computational complexity of approximately $\mathcal{O}(C^2 d + T C d)$. For long sequences where the number of time steps $T$ is much larger than the number of channels $C$, this is significantly more scalable than standard temporal self-attention, which has a complexity of $\mathcal{O}(T^2 d)$.

\subsection{Adaptive Multi-modal Ensemble}
The model processes the input signal through multiple, parallel feature extraction modules, including the Mamba-based model, multi-scale temporal convolutions, and the frequency-domain module. To synthesize insights from these diverse perspectives, an adaptive ensemble is employed that dynamically weights the contribution of each modality at every time step. This is crucial, as the predictive power of any single modality can fluctuate depending on the local characteristics of the time series.

First, each modality $m$ produces a feature representation, which is denoted generically as $O^m \in \mathbb{R}^{T \times d}$. This representation is then passed to a dedicated prediction head, $f_m$, to generate a modality-specific forecast:
\begin{equation}
    \hat{y}_t^m = f_m(O_{t,:}^m).
\end{equation}
Here, the colon (:) in the subscript of $O_{t,:}^m$ is a standard tensor notation indicating that all elements along the corresponding dimension are selected. Specifically, $O_{t,:}^m$ extracts the complete feature vector of size $d$ for modality $m$ at time step $t$.

The core of the ensemble is the mechanism for generating the time-varying weights $w_t^m$, which functions as a form of attention mechanism. A context vector $c_t \in \mathbb{R}^{d_c}$ is first constructed that summarizes the state of the system at time $t$ by aggregating information from all modalities. This is done by pooling the feature representations (e.g., via global average pooling), concatenating them, and processing the result through an MLP:
\begin{equation}
    c_t = \operatorname{MLP}\left(\operatorname{concat}_m\left(\operatorname{pool}_t(O^m)\right)\right).
\end{equation}
This context vector $c_t$ acts as a dynamic query. A gating function $g_m$ (e.g., a linear layer) computes a relevance score $s_t^m = g_m(c_t)$ for each modality. The scores are normalized using a softmax function to produce the final weights, $w_t^m = \operatorname{softmax}_m(s_t^m)$, which are positive and sum to one.

The final ensemble prediction $\hat{y}_t$ is the weighted average of the individual modality predictions, conceptually similar to a Mixture of Experts (MoE) model:
\begin{equation}
    \hat{y}_t = \sum_m w_t^m \hat{y}_t^m.
\end{equation}

Finally, to improve robustness under non-stationarity, regularization is applied to the weights $w_t^m$. Entropy regularization helps to mitigate overconfident collapse (where the model relies on a single modality), while temporal smoothness regularization penalizes erratic oscillations in the weights over time. This improves the stability and interpretability of the ensemble.

\subsection{Uncertainty Estimation}
For point forecasts, the model is trained by minimizing a standard regression loss function between the predicted mean $\hat{y}_t$ and the true value $y_t$. Either the Mean Squared Error (MSE) loss or the Huber loss is used.

The MSE loss is defined as:
\begin{equation}
    \mathcal{L}_{\text{MSE}} = \frac{1}{T} \sum_{t=1}^{T} (y_t - \hat{y}_t)^2.
\end{equation}

The Huber loss provides a balance between MSE and Mean Absolute Error (MAE), making it more robust to outliers. For an error term $a_t = y_t - \hat{y}_t$, it is defined as:
\begin{equation}
    \mathcal{L}_{\text{Huber}} = \frac{1}{T} \sum_{t=1}^{T} L_\delta(a_t),
\end{equation}
where $L_\delta(a)$ is a piecewise function:
\begin{equation}
    L_\delta(a) = 
    \begin{cases} 
    \frac{1}{2}a^2 & \text{for } |a| \le \delta, \\
    \delta(|a| - \frac{1}{2}\delta) & \text{otherwise}.
    \end{cases}
\end{equation}
The hyperparameter $\delta$ controls the transition point between the quadratic and linear parts of the loss.

To incorporate uncertainty estimation, the final prediction layer is modified. Instead of a single output, the model employs a dedicated uncertainty head that, given the final feature representation $h_t \in \mathbb{R}^{d_{\text{model}}}$ for time step $t$, produces both a mean prediction $\hat{y}_t$ and a standard deviation $\sigma_t$. This is achieved through two separate linear projections:
\begin{align}
    \hat{y}_t &= W_{\mu}h_t + b_{\mu}, \quad \sigma_t &= \text{Softplus}(W_{\sigma}h_t + b_{\sigma}),
\end{align}
where $W_{\mu}, W_{\sigma}$ and $b_{\mu}, b_{\sigma}$ are learnable parameters of their respective linear layers. The Softplus activation function, defined as $\text{Softplus}(x) = \log(1 + \exp(x))$, is applied to the standard deviation output to ensure its positivity.

When uncertainty estimation is enabled, the training objective is adjusted. The total loss combines the point-prediction loss (either MSE or Huber) with a regularization term on the predicted standard deviations to encourage confidence and ensure stable training:
\begin{equation}
    \mathcal{L}_{\text{total}} = \mathcal{L}_{\text{point}}(y, \hat{y}) + \lambda \frac{1}{T} \sum_{t=1}^{T} \sigma_t,
\end{equation}
where $\mathcal{L}_{\text{point}}$ is either $\mathcal{L}_{\text{MSE}}$ or $\mathcal{L}_{\text{Huber}}$, and $\lambda$ is a hyperparameter that controls the strength of the uncertainty regularization.

\section{Architectural Design Principles}
In this section, a theoretical rationale is provided for MAESTRO's architecture. The argument is that its composition of spectro-temporal modules is not an arbitrary ensemble but a principled design tailored to the challenges of epidemiological forecasting: multi-scale seasonality, noisy data, and long-range dependencies. The design philosophy is to assign specific, complementary roles to each component, guided by their theoretical properties.

\subsection{Multi-Scale Seasonality Decomposition}
Epidemiological data exhibits seasonality at multiple scales (e.g., annual cycles). A key design principle of MAESTRO is to explicitly decompose the signal to handle this structure, a technique proven effective in recent long-term forecasting models\cite{wu2021autoformer, zeng2023dlinear}.

\textbf{Methodology:} The model employs an additive decomposition strategy, inspired by classical methods like STL (Seasonal-Trend decomposition based on Loess)~\cite{cleveland1990stl}. Both approaches share the foundational goal of separating a time series $x_t$ into trend and seasonal components. However, they differ significantly in their methodology for achieving this separation.

A classical STL decomposition follows an iterative procedure to robustly separate a series $x_t$ into a trend $T_t$, a seasonal $S_t$, and a remainder component $R_t$, such that $x_t = T_t + S_t + R_t$. The core of STL lies in its use of Loess (Locally Estimated Scatterplot Smoothing), a non-parametric regression method. The process involves an inner loop for seasonal and trend smoothing and an outer loop that computes robustness weights to mitigate the effect of outliers. Formally, at each outer loop iteration $j$, the trend and seasonal components are updated via a sequence of Loess smoothing operations, denoted by the operator $\mathcal{L}$:
\begin{align}
    S_t^{(j)} &= \mathcal{L}_{\text{seasonal}}(x_t - T_t^{(j-1)}) \\
    T_t^{(j)} &= \mathcal{L}_{\text{trend}}(x_t - S_t^{(j)})
\end{align}
This iterative, weighted regression makes STL robust but computationally intensive.

In contrast, MAESTRO adopts the decomposition principle but replaces the iterative Loess smoother with a single-pass, linear moving-average filter for computational efficiency within a deep learning framework. This filter, $\mathcal{D}_K$, is implemented as a 1D average pooling layer. The decomposition is direct and non-iterative:
\begin{align}
    x^{(\mathrm{tr})}_t &= \mathcal{D}_K(x_t) = \text{AvgPool1d}(\text{Padding}(x_t), \text{kernel\_size}=K, \text{stride}=1) \\
    x^{(\mathrm{se})}_t &= x_t - x^{(\mathrm{tr})}_t
\end{align}
While this approach is less robust to outliers than STL, its computational efficiency and differentiability make it highly suitable for end-to-end training of large-scale models. It effectively captures the low-frequency trend, aligning with the primary goal of decomposition in forecasting models~\cite{wu2021autoformer, zeng2023dlinear}.

As shown in Lemma~\ref{lem:ma-frequency}, the moving average acts as a low-pass filter. This fixed decomposition is complemented by a \textit{learnable} frequency-domain module inspired by\cite{zhou2022fedformer}, which applies a spectral filter to the input $x$:
\begin{equation}
    \mathcal{T}_M(x) = \mathcal{F}^{-1}(M \odot \mathcal{F}(x))
\end{equation}
Here, $\mathcal{F}$ is the Fourier Transform, and $M(\omega)$ is a learnable complex mask that shapes the signal's spectrum, allowing the model to isolate or suppress specific frequency bands corresponding to seasonal patterns.

\textbf{Theoretical Justification:} This two-stage approach is principled because it combines a stable, fixed-basis decomposition with a flexible, data-driven one. The moving average provides a robust, low-variance estimate of the trend, while the learnable frequency filter offers high expressiveness. Proposition~\ref{prop:spectral-stability} guarantees that this learnable filtering operation is stable, as the Lipschitz continuity ensures that small input perturbations do not lead to unbounded output changes. This is critical for robustness when dealing with noisy surveillance data.

\textbf{Synergy:} This synergy between a time-domain decomposition (via moving average) and a frequency-domain filtering (via the learnable mask) allows the model to factorize the learning problem effectively. The trend and seasonal components can be modeled by specialized downstream modules, preventing any single component from being overwhelmed by the signal's complexity. This decomposition is a core principle that enhances both the accuracy and robustness of the overall model.
\begin{lemma}[Moving-Average Frequency Response \cite{box2015time}]
\label{lem:ma-frequency}
Let $\mathcal{D}_K$ be the length-$K$ moving-average operator with impulse response $h[n]=\tfrac{1}{K}\sum_{m=0}^{K-1}\delta[n-m]$. Here, the symbols are defined as:
\begin{itemize}
    \item $\mathcal{D}_K$: A moving-average operator that smooths a time series by averaging over a window of length $K$. It is a form of a low-pass filter.
    \item $K$: A positive integer representing the size of the averaging window. A larger $K$ results in more smoothing.
    \item $h[n]$: The impulse response of the operator at discrete time step $n$. It defines the output when the input is a single impulse.
    \item $\delta[n-m]$: The Kronecker delta function, which is 1 if $n=m$ and 0 otherwise. It is used here to define the rectangular window of the moving average.
\end{itemize}
Its discrete-time frequency response is:
\begin{equation}
H(\omega)=\frac{1}{K}\,e^{-\mathrm{i}\omega (K-1)/2}\,\frac{\sin(K\omega/2)}{\sin(\omega/2)},\quad \omega\in(-\pi,\pi].
\end{equation}
In this equation:
\begin{itemize}
    \item $H(\omega)$: The frequency response, a complex function describing how the operator modifies the amplitude and phase of each frequency component $\omega$ of the input signal.
    \item $\omega$: The angular frequency in radians per sample, with the domain $(-\pi, \pi]$ representing the unique range of frequencies for a discrete-time signal.
    \item $\mathrm{i}$: The imaginary unit, satisfying $\mathrm{i}^2 = -1$.
    \item $|H(\omega)|$: The magnitude of the frequency response, indicating the gain (amplification or attenuation) applied by the filter at frequency $\omega$.
\end{itemize}
In particular, $|H(\omega)|$ is maximized at $\omega=0$ and decays as $|\omega|$ increases, confirming the operator's low-pass behavior.
\end{lemma}

\begin{proposition}[Spectral Filter Stability]
\label{prop:spectral-stability}
Let $M: \mathbb{C} \to \mathbb{C}$ be a complex spectral mask learned by the frequency-domain module, satisfying $\sup_\omega \|M(\omega)\| \leq G_f$ for some constant $G_f > 0$ (enforced via parameterization and clipping). The frequency-domain filtering operator $\mathcal{T}_M: x \mapsto \mathcal{F}^{-1}(M \odot \mathcal{F}(x))$ is $G_f$-Lipschitz continuous, i.e.,
\begin{equation}
\|\mathcal{T}_M(x_1) - \mathcal{T}_M(x_2)\|_2 \leq G_f \|x_1 - x_2\|_2
\end{equation}
for all input sequences $x_1, x_2 \in \mathbb{R}^T$. The symbols are defined as:
\begin{itemize}
    \item $M(\omega)$: A complex-valued spectral mask. This is a function learned by the model that is applied in the frequency domain to selectively amplify or attenuate different frequency components of the signal.
    \item $\mathbb{C}$: The set of complex numbers. The mask is complex to allow for modifications to both amplitude and phase.
    \item $G_f$: A positive constant that serves as an upper bound for the norm of the spectral mask across all frequencies. This is crucial for ensuring the filtering process is stable and does not lead to exploding output values.
    \item $\sup_\omega$: The supremum, or the least upper bound, taken over all frequencies $\omega$.
    \item $\mathcal{T}_M$: The filtering operator that encapsulates the entire process: transforming a signal to the frequency domain, applying the mask, and transforming it back to the time domain.
    \item $x$: An input time-series signal in $\mathbb{R}^T$.
    \item $\mathcal{F}$ and $\mathcal{F}^{-1}$: The Discrete Fourier Transform (DFT) and its inverse (IDFT), which convert a signal between the time and frequency domains.
    \item $\odot$: The Hadamard product, which performs element-wise multiplication. Here, it applies the mask $M$ to the Fourier-transformed signal.
    \item $\|\cdot\|_2$: The Euclidean ($L_2$) norm, used to measure the length or energy of a sequence.
    \item $G_f$-Lipschitz continuity: This is a strong guarantee of stability. It means that the operator $\mathcal{T}_M$ will not amplify the "distance" (i.e., the norm of the difference) between any two input signals by more than the factor $G_f$. This prevents small perturbations in the input from causing large changes in the output.
    \item $\mathbb{R}^T$: The space of all real-valued sequences of length $T$.
\end{itemize}
\end{proposition}

\begin{proof}
The proof demonstrates the Lipschitz continuity by leveraging Parseval's theorem, which states that the energy of a signal is conserved under the Fourier transform.
\begin{align}
\|\mathcal{T}_M(x_1) - \mathcal{T}_M(x_2)\|_2^2 &= \|\mathcal{F}^{-1}(M \odot (\mathcal{F}(x_1) - \mathcal{F}(x_2)))\|_2^2 \\
&= \|M \odot (\mathcal{F}(x_1) - \mathcal{F}(x_2))\|_2^2 \\
&\leq G_f^2 \,\|\mathcal{F}(x_1) - \mathcal{F}(x_2)\|_2^2 \\
&= G_f^2 \,\|x_1 - x_2\|_2^2.
\end{align}
Taking the square root of both sides completes the proof. The derivation proceeds as follows:
\begin{enumerate}
    \item The definition of the operator $\mathcal{T}_M$ is applied to the difference between two outputs.
    \item By Parseval's theorem, the $L_2$ norm of a signal is equal to the $L_2$ norm of its Fourier transform (up to a scaling constant that cancels out), so the $\mathcal{F}^{-1}$ can be removed.
    \item The norm of the element-wise product is less than or equal to the maximum of the mask's norm multiplied by the norm of the other term. Since $\|M(\omega)\| \leq G_f$ for all $\omega$, this inequality holds.
    \item Applying Parseval's theorem again, the norm of the Fourier transform of the difference is equal to the norm of the difference in the original time domain.
\end{enumerate}
\end{proof}

\subsection{Hybrid Modeling of Long-Range and Aperiodic Dynamics}
After decomposing multi-scale seasonality, the model must capture the remaining complex dynamics, which include long-range dependencies and aperiodic events. MAESTRO employs a principled hybrid of State-Space Models (Mamba) and attention to address this.

\textbf{Methodology:} The core of the design is a synergistic combination of Mamba blocks and a multi-scale Transformer encoder. The Mamba block is based on a continuous-time State-Space Model (SSM) defined by:
\begin{align}
    h'(t) &= A h(t) + B x(t), \quad y(t) = C h(t) + D x(t)
\end{align}
where $A, B, C, D$ are learnable matrices. This is discretized for computation. Concurrently, the multi-scale Transformer encoder uses self-attention to identify discrete, content-based relationships across time:
\begin{equation}
    \text{Attention}(Q, K, V) = \text{softmax}\left(\frac{QK^T}{\sqrt{d_k}}\right)V
\end{equation}
Mamba models the continuous evolution of the time series, while attention captures discrete, event-driven interactions.

\textbf{Theoretical Justification.} This hybrid approach is principled because it assigns distinct modeling roles based on the theoretical strengths of each component.
\begin{enumerate}
  \item State-Space Models for Continuous Dynamics. The Mamba block is built on a continuous-time State-Space Model (SSM), which is discretized for computation \cite{gu2023mamba}. Its strength lies in efficiently compressing long historical context into a fixed-size state, which evolves over time. The selective gating mechanism, which can be viewed as a time-varying EMA \cite{gu2023mamba}, allows the model to dynamically adjust its memory horizon. This makes Mamba exceptionally well-suited for modeling the smooth, slowly evolving residual trends and aperiodic dynamics that remain after the primary seasonal components have been removed by the decomposition and frequency-domain modules.

  \item Attention for Discrete Events. In contrast, the self-attention mechanism excels at capturing sparse, content-addressable relationships. It can identify and connect specific, high-impact historical events (e.g., a past outbreak with a similar signature) regardless of how far apart they are in time. This is crucial for modeling sharp, non-local interactions that a purely sequential model like an SSM might miss.
  
\end{enumerate}

\textbf{Synergy:} The combination of Mamba and attention creates a powerful and efficient factorization of the learning problem. Mamba provides a robust, computationally efficient backbone for modeling the global, continuous evolution of the system's state. Attention complements this by providing a high-resolution mechanism to focus on specific, critical events in the past. This division of labor-continuous context from Mamba, discrete events from attention-allows MAESTRO to model complex epidemiological dynamics more effectively than either component could alone, providing a principled foundation for its architecture.

\subsection{Expressiveness of State-Space and Decomposition}
The model's ability to handle long dependencies and complex structures is further enhanced by the properties of its core building blocks.

\begin{lemma}[State-Space Memory \cite{gu2023mamba}]
The discrete state transition matrix $A^{\mathrm{d}}=\exp(A\Delta_t)$ of the Mamba block, with $A$ parameterized to have negative real parts, has a spectral radius $\rho(A^{\mathrm{d}})$ that is close to but less than 1. This ensures that the state $h_t$ maintains a long-range memory of past inputs, with a forgetting rate controlled by the learnable step size $\Delta_t$. This is critical for capturing long-term trends and dependencies in influenza dynamics.
\end{lemma}

\begin{lemma}[SSM Discretization Stability \cite{gu2022s4}]
\label{lem:ssm-stability}
A fundamental property of SSMs is that a continuous linear SSM with a stable system matrix $A$ (all eigenvalues have negative real parts) results in a stable discrete system. Specifically, for a discretization step size $\Delta > 0$, the discrete system matrix $A^d = \exp(A\Delta)$ has a spectral radius $\rho(A^d) < 1$, ensuring asymptotic stability.
\end{lemma}

Following the successful application of decomposition in long-term forecasting \cite{wu2021autoformer, zhou2022fedformer}, the learning task is also simplified by initially decomposing the input $x$ into a trend component $x^{(\mathrm{tr})} = \mathcal{D}_K(x)$ and a seasonal component $x^{(\mathrm{se})} = x - x^{(\mathrm{tr})}$. The trend, as a low-frequency signal, can be effectively modeled by simpler network components. This allows the high-capacity encoders to focus on the more complex, high-frequency seasonal variations. Such separation can be viewed as a form of curriculum learning, which can improve model stability and sample efficiency.

\begin{proposition}[Multi-Scale Convolution Equivalence]
\label{prop:multiscale-equiv}
Let $\{C_{k,\delta}\}$ be a collection of 1D convolution operators with kernel sizes $k \in \mathcal{K}$ and dilations $\delta \in \mathcal{D}$. The attention-weighted aggregation
\begin{equation}
U = \sum_{k,\delta} \alpha_{k,\delta} C_{k,\delta}(H)
\end{equation}
with $\sum_{k,\delta} \alpha_{k,\delta} = 1$ and $\alpha_{k,\delta} \geq 0$ is equivalent to a single convolution with an adaptive kernel that linearly combines the original kernels. The symbols are defined as:
\begin{itemize}
    \item $C_{k,\delta}$: A 1D convolution operator with kernel size $k$ and dilation $\delta$. Dilation controls the spacing between kernel elements, allowing the operator to capture patterns at different temporal scales.
    \item $\mathcal{K}$: The set of kernel sizes used in the multi-scale convolution, for capturing patterns from fine-grained to coarse-grained temporal dependencies.
    \item $\mathcal{D}$: The set of dilation rates, which exponentially expand the receptive field without increasing computational cost.
    \item $\alpha_{k,\delta}$: Attention weights for each convolution operator $(k,\delta)$, learned dynamically based on input content. These weights determine the relative importance of different temporal scales.
    \item $H$: Input feature tensor of shape $(B, T, d)$ where $B$ is batch size, $T$ is sequence length, and $d$ is feature dimension.
    \item $U$: Output feature tensor after multi-scale convolution aggregation, maintaining the same shape as $H$.
    \item $w_{k,\delta}$: The learnable convolution kernel weights for operator $C_{k,\delta}$.
\end{itemize}
\end{proposition}

\begin{proof}
By the linearity property of convolution operations, the weighted sum of convolutions can be rewritten as:
\begin{align}
U &= \sum_{k,\delta} \alpha_{k,\delta} C_{k,\delta}(H) \\
&= \sum_{k,\delta} \alpha_{k,\delta} (w_{k,\delta} \ast H) \\
&= \left(\sum_{k,\delta} \alpha_{k,\delta} w_{k,\delta}\right) \ast H
\end{align}
where $\ast$ denotes the convolution operation. The term $\sum_{k,\delta} \alpha_{k,\delta} w_{k,\delta}$ represents an adaptive kernel that is a linear combination of the original kernels, weighted by the attention coefficients.
\end{proof}

\subsection{Uncertainty Estimation}
With mean $\mu$ and standard deviation $\sigma=\operatorname{softplus}(g(x))+\varepsilon$, assuming conditional Gaussian likelihood, the per-step negative log-likelihood is
\begin{equation}
\mathcal{L}_{\mathrm{NLL}} = \tfrac{1}{2}\log(2\pi\sigma^2)+\tfrac{(y-\mu)^2}{2\sigma^2}.
\end{equation}
The model is trained with a composite objective $\mathcal{L} = \lambda_1\,\operatorname{MSE} + \lambda_2\,\operatorname{NLL} + \lambda_3\,R_{\text{spectral}} + \lambda_4\,R_{\text{weights}} + \lambda_5\,R_{\text{stability}}$, where $R_{\text{spectral}}$ includes magnitude and smoothness penalties on $M(\omega)$; $R_{\text{weights}}$ regularizes ensemble weights (entropy/smoothness); and $R_{\text{stability}}$ constraints SSM spectra.

\begin{lemma}[Uncertainty Calibration \cite{kuleshov2018accurate}]
\label{lem:uncertainty-calibration}
Under the Gaussian likelihood assumption, the predicted uncertainty $\sigma$ is well-calibrated if the standardized residuals $(y - \mu)/\sigma$ follow a standard normal distribution. This can be verified empirically through Q-Q plots and the Kolmogorov-Smirnov test.
\end{lemma}

\section{Data Overview}
\label{sec:data-overview}
To provide a comprehensive overview of the datasets utilized in this study, a series of visualizations are presented that encapsulate their primary characteristics. These visualizations are crucial for understanding the underlying data distributions and temporal patterns that the proposed model, MAESTRO, aims to learn.

\begin{figure}[htbp]
    \centering
    \includegraphics[width=0.9\textwidth]{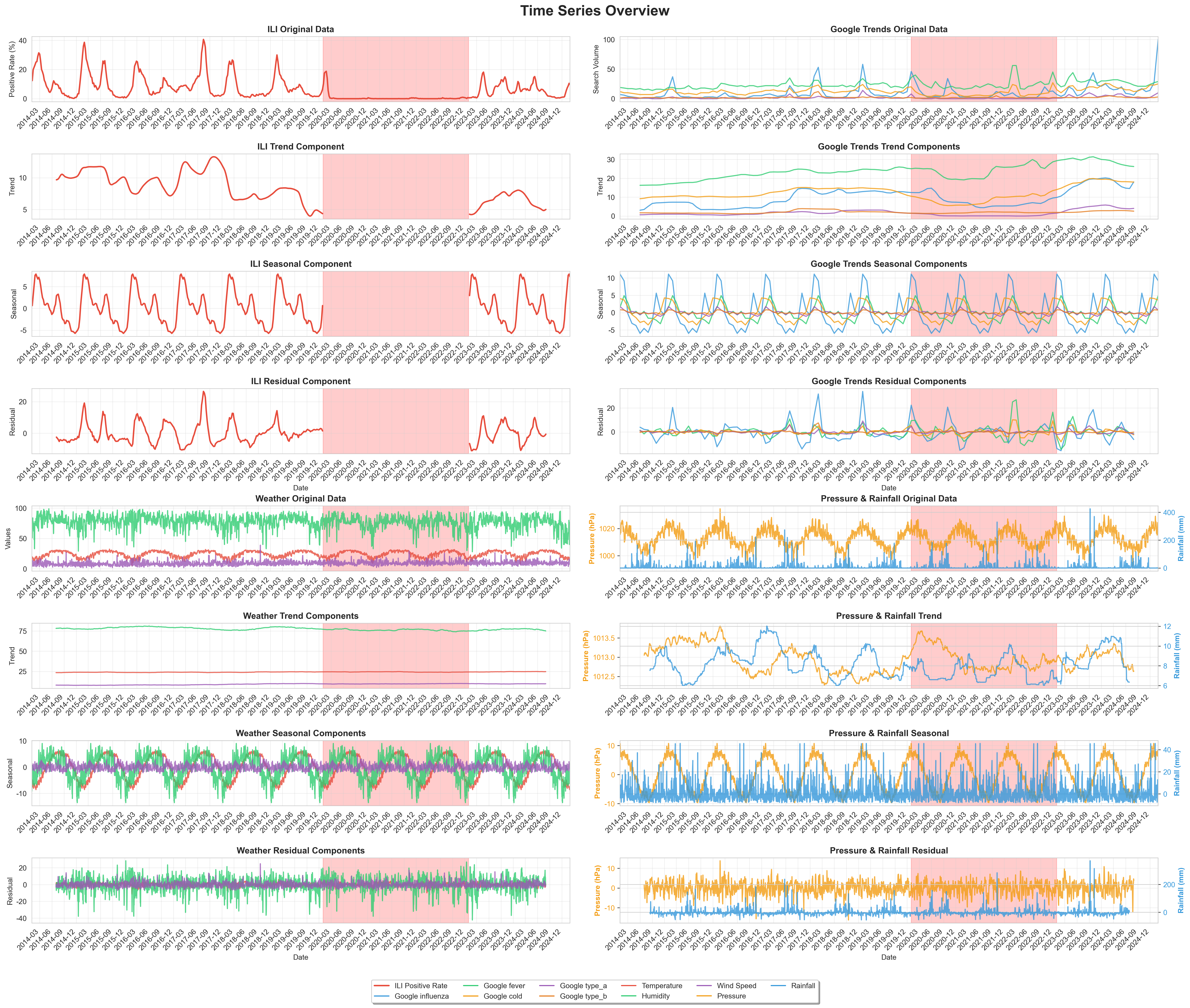}
    \caption{Time-series overview of influenza-like illness (ILI) cases, Google Trends data, and weather metrics. This figure illustrates the temporal dynamics and potential correlations between the different data modalities.}
    \label{fig:time-series-overview}
\end{figure}

\begin{figure}[htbp]
    \centering
    \includegraphics[width=0.9\textwidth]{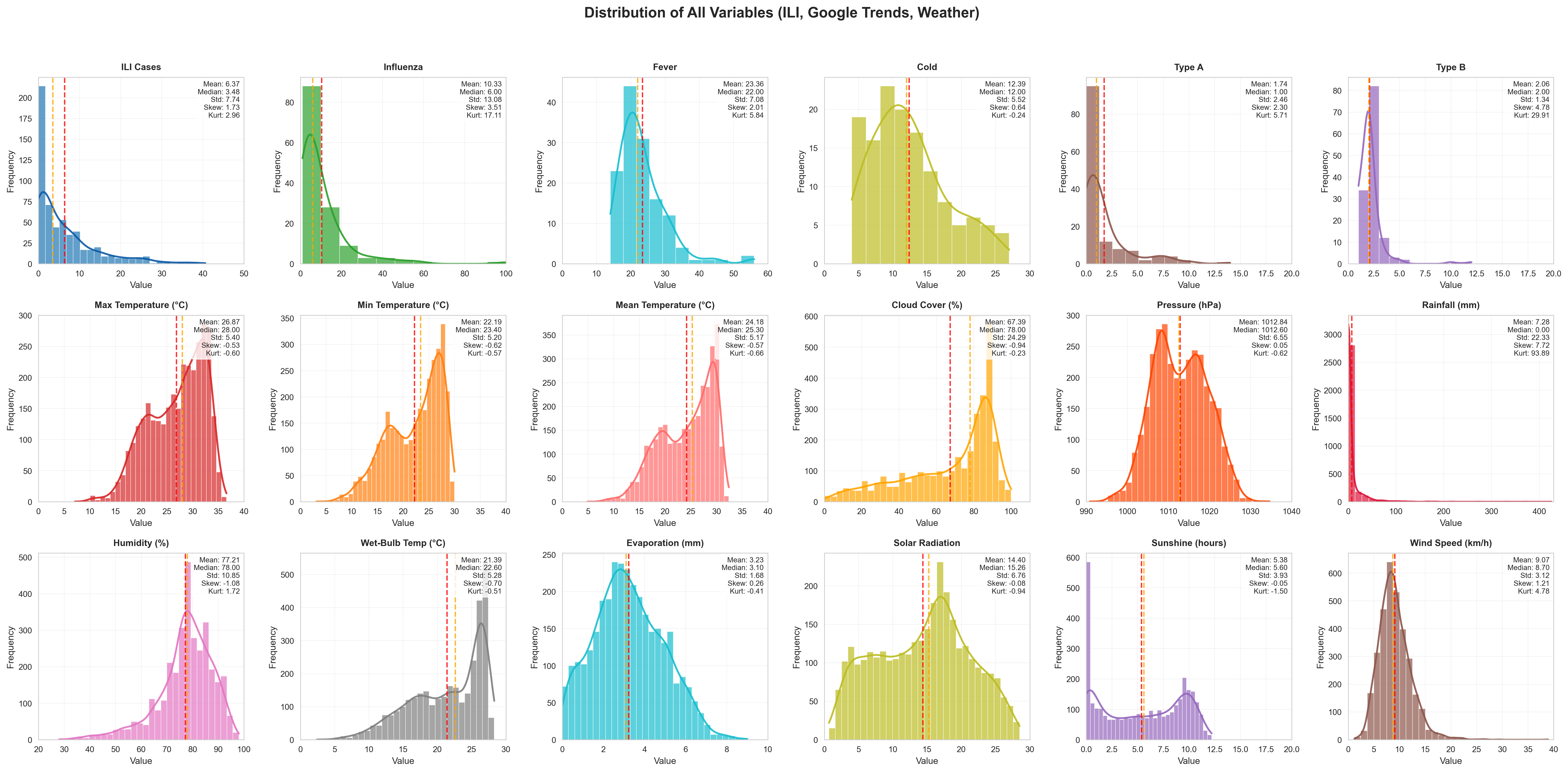}
    \caption{Comprehensive distribution analysis of all input variables. The figure displays histograms and Kernel Density Estimation (KDE) plots for Influenza-Like Illness (ILI) cases, Google Trends keywords, and meteorological data. Key statistical metrics, including mean, median, standard deviation, skewness, and kurtosis, are annotated for each subplot. The visualization reveals the distinct distributional characteristics of each data modality, such as the pronounced right-skew of ILI and Google Trends data, and the diverse patterns observed in weather variables.}
    \label{fig:combined-dist}
\end{figure}

The visualizations in Figure~\ref{fig:time-series-overview} and Figure~\ref{fig:combined-dist} offer a foundational understanding of the data's structure. Figure~\ref{fig:time-series-overview} provides a macro-level view of the temporal trends, clearly revealing seasonal and residual components that justify the model's decomposition-based architecture. Concurrently, Figure~\ref{fig:combined-dist} details the statistical distributions, highlighting the pronounced right-skew and high kurtosis in ILI and Google Trends data, which underscores the need for robust loss functions and uncertainty estimation. These insights are instrumental in appreciating the complexity of the forecasting task and the design of the model.

\section{Experimental Setup}
Chronological splits are adopted for the dataset, partitioning it into training (60\%), validation (20\%), and testing (20\%) sets. The mean performance over 5 independent training runs is reported, each with a different random seed, to ensure the stability of the results. For fairness, look-back $L$, horizon $H$, optimizer, epochs, and early stopping are aligned across models. Modalities are concatenated for single-stream baselines; models with native multi-branch fusion are also evaluated in that mode when feasible.
\paragraph{Implementation.} PyTorch with modular components (encoders, Mamba block, frequency module, cross-channel attention, temporal projection, uncertainty estimator). Default: $L=30$, $H=1$, Adam\cite{kingma2015adam} with learning rate 1e-3, batch size 256, early stopping, ReduceLROnPlateau.
\paragraph{Baselines and SOTA.} Classical: ARIMA; ML: SVR; RNN: LSTM, GRU; Transformers: Transformer, Informer, Autoformer, FEDformer, iTransformer; Spectro/Conv: TimesNet; Linear head: DLinear. Official or widely used implementations are used with consistent preprocessing and hyperparameter budgets.
\paragraph{Metrics.} MAE, RMSE, MAPE, and R\textsuperscript{2}.

\section{Results}

\subsection{Main Results}
\label{sec:main-results}

The primary results of the forecasting experiments are presented in Figure~\ref{fig:main-predictions}. This figure provides a comprehensive visualization of MAESTRO's predictive performance against the ground truth, alongside several state-of-the-art baseline models. The visualization highlights the model's ability to accurately capture the complex temporal patterns and fluctuations present in the influenza-like illness (ILI) data.

\begin{figure*}[htbp]
    \centering
    \includegraphics[width=\textwidth]{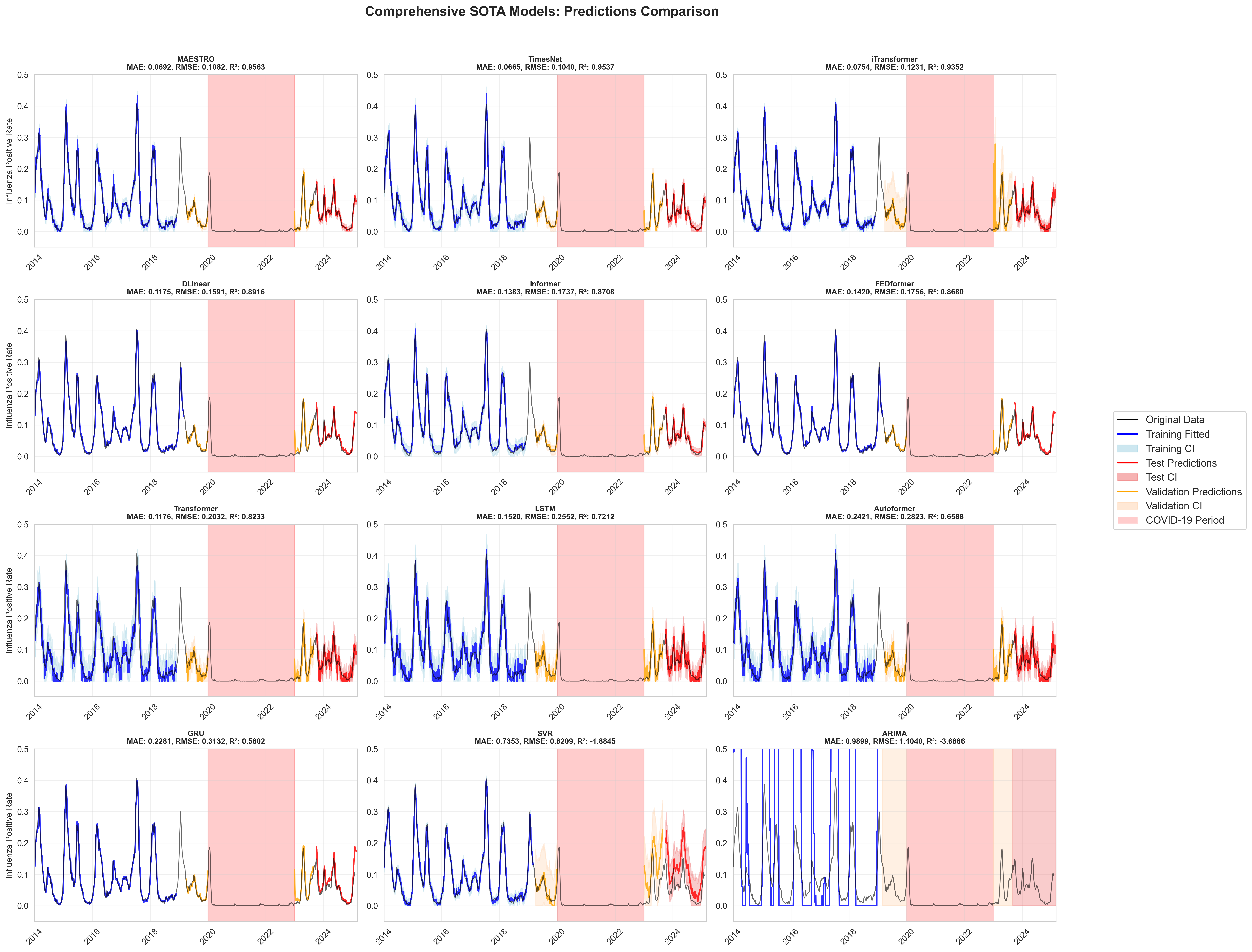}
    \caption{Comprehensive prediction results of MAESTRO and baseline models. This figure illustrates the model's forecasting accuracy over the test set, showcasing how it closely tracks the ground truth. See Table~\ref{tab:sota} for detailed performance metrics.}
    \label{fig:main-predictions}
\end{figure*}

As depicted in Figure~\ref{fig:main-predictions}, MAESTRO's predictions closely follow the actual ILI case counts, demonstrating its robustness and high fidelity. While some models like TimesNet show marginal advantages in specific error metrics such as MAE and RMSE, MAESTRO achieves the highest coefficient of determination (R\textsuperscript{2}), indicating a superior overall fit to the data's variance, as detailed in Table~\ref{tab:sota}. The shaded areas represent the confidence intervals, indicating the model's uncertainty estimates, which are crucial for reliable decision-making in public health scenarios.

\subsection{Performance Comparison}
\label{sec:performance-comparison}

To provide a clear and quantitative comparison of MAESTRO's performance against baseline models, the key performance metrics are presented in Table~\ref{tab:sota}, including Mean Absolute Error (MAE), Root Mean Squared Error (RMSE), Mean Absolute Percentage Error (MAPE), and the coefficient of determination (R\textsuperscript{2}).

\begin{table*}[!t]
  \centering
  \caption{Comparison with state-of-the-art models on Hong Kong influenza forecasting (higher is better for R\textsuperscript{2}, lower otherwise). Best and second-best are \textbf{bold} and \underline{underlined}.}
  \label{tab:sota}
  \begin{tabular}{lcccc}
    \toprule
    Method & MAE & RMSE & MAPE (\%) & R\textsuperscript{2} \\
    \midrule
    \textbf{MAESTRO (ours)} & \underline{0.069} & \underline{0.108} & \underline{74.08} & \textbf{0.956} \\
    TimesNet & \textbf{0.067} & \textbf{0.104} & 80.10 & \underline{0.954} \\
    iTransformer & 0.075 & 0.123 & \textbf{73.77} & 0.935 \\
    DLinear & 0.118 & 0.159 & 152.73 & 0.892 \\
    Informer & 0.138 & 0.174 & 81.47 & 0.871 \\
    FEDformer & 0.142 & 0.176 & 91.44 & 0.868 \\
    Transformer & 0.118 & 0.203 & 172.05 & 0.823 \\
    LSTM & 0.152 & 0.255 & 137.03 & 0.721 \\
    Autoformer & 0.242 & 0.282 & 139.13 & 0.659 \\
    GRU & 0.228 & 0.313 & 148.43 & 0.580 \\
    SVR & 0.735 & 0.821 & 503.86 & -1.884 \\
    ARIMA & 0.990 & 1.104 & 479.65 & -3.689 \\
    \bottomrule
  \end{tabular}
\end{table*}

\subsection{Ablation Studies}
\paragraph{Terminology and Variants.}
Consistent naming is followed between paper and code for ablations:
\begin{itemize}
    \item w/o Mamba: remove the state-space (Mamba) enhancement block while keeping the rest of the architecture intact.
    \item w/o Frequency Domain: disable the frequency-domain filtering module (FFT/filter/iFFT path).
    \item w/o Cross-Channel Attention: drop cross-channel attention for modality interaction.
    \item Modality ablations: single\_modal\_flu, single\_modal\_trend, single\_modal\_weather (use only one modality); modality\_flu\_trend, modality\_flu\_weather, modality\_trend\_weather (use two); full\_modal (use all).
    \item Architecture baselines: Transformer, GRU, LSTM, and Linear Regression baselines trained under the same protocol.
\end{itemize}

\begin{table*}[!t]
  \centering
  \caption{Ablation study results. MAESTRO (Full Model) is compared against variants with specific components removed or isolated. Performance metrics (MAE, RMSE, R\textsuperscript{2}) and parameter counts are reported. Lower is better for MAE/RMSE, higher for R\textsuperscript{2}.}
  \label{tab:ablation}
  \begin{tabular}{lrrrr}
    \toprule
    Model Variant & MAE & RMSE & R\textsuperscript{2} & Trainable Params (M) \\
    \midrule
    \textbf{MAESTRO (Full Model)} & \textbf{0.064} & \textbf{0.101} & \textbf{0.962} & \textbf{1.036} \\
    \midrule
    \textit{Component Ablations} & & & & \\
    w/o Mamba & 0.080 & 0.128 & 0.939 & 1.235 \\
    w/o Cross-Channel Attention & 0.065 & 0.114 & 0.952 & 0.995 \\
    w/o Frequency Domain & 0.075 & 0.120 & 0.946 & 1.055 \\
    w/o Decomposition & 0.075 & 0.121 & 0.946 & 1.303 \\
    w/o Multi-scale & 0.073 & 0.109 & 0.955 & 0.293 \\
    \midrule
    \textit{Minimalist Baselines} & & & & \\
    only Mamba & 0.101 & 0.153 & 0.912 & 0.373 \\
    only Cross-Attention & 0.101 & 0.153 & 0.912 & 0.342 \\
    only Frequency Domain & 0.086 & 0.138 & 0.929 & 0.274 \\
    Minimal Model & 0.120 & 0.182 & 0.876 & 0.274 \\
    \bottomrule
  \end{tabular}
\end{table*}

\begin{figure}[htbp]
  \centering
  \includegraphics[width=0.98\linewidth]{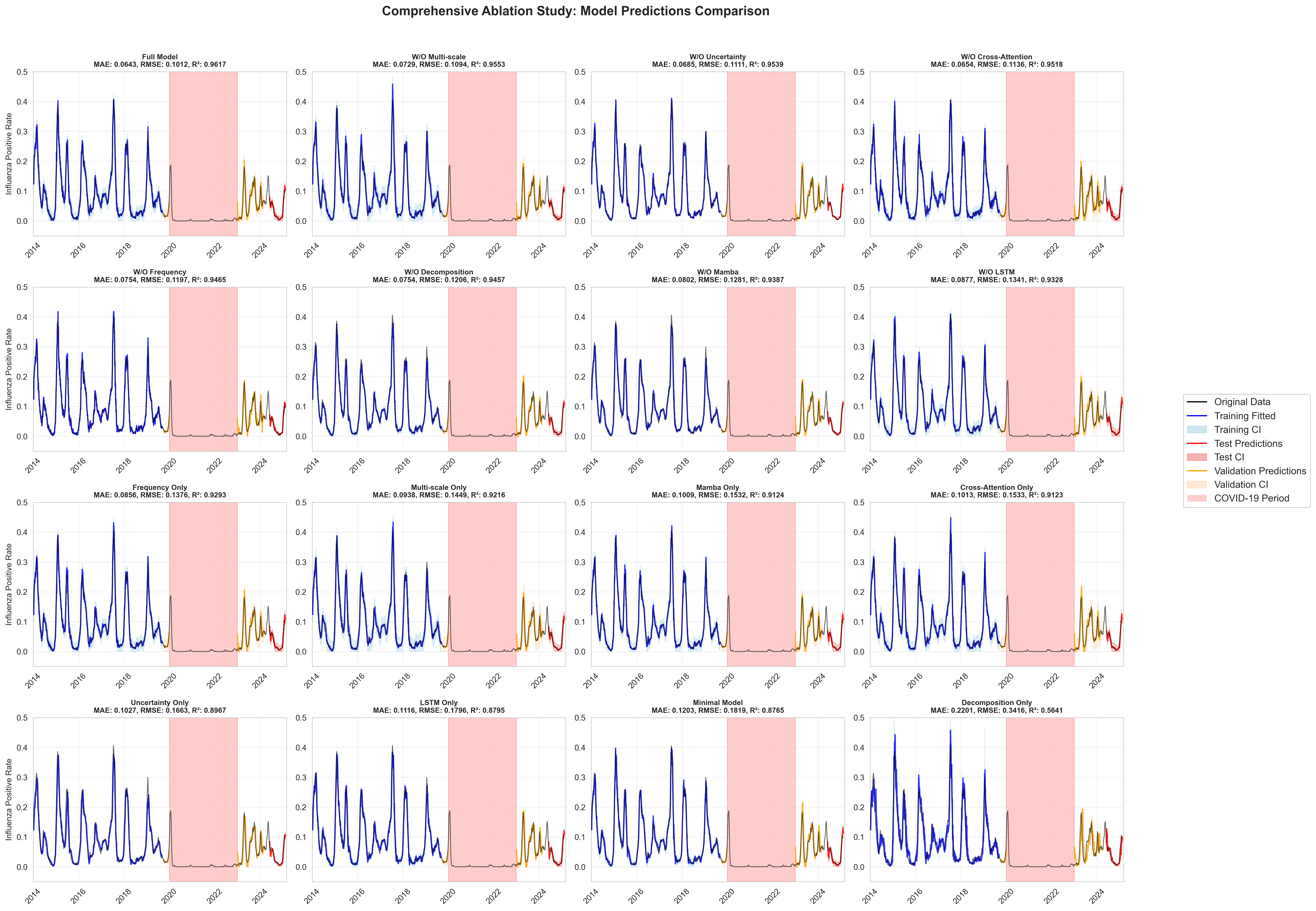}
  \caption{Comprehensive prediction comparison across ablations and baselines. The red dashed line separates training and testing periods; R\textsuperscript{2} is annotated in each panel.}
  \Description{Multi-panel figure showing original, fitted, and predicted influenza time series across MAESTRO variants and baselines, with a red dashed split line and per-panel R-squared labels.}
  \label{fig:comprehensive}
\end{figure}

\begin{figure}[htbp]
    \centering
    \includegraphics[width=0.95\textwidth]{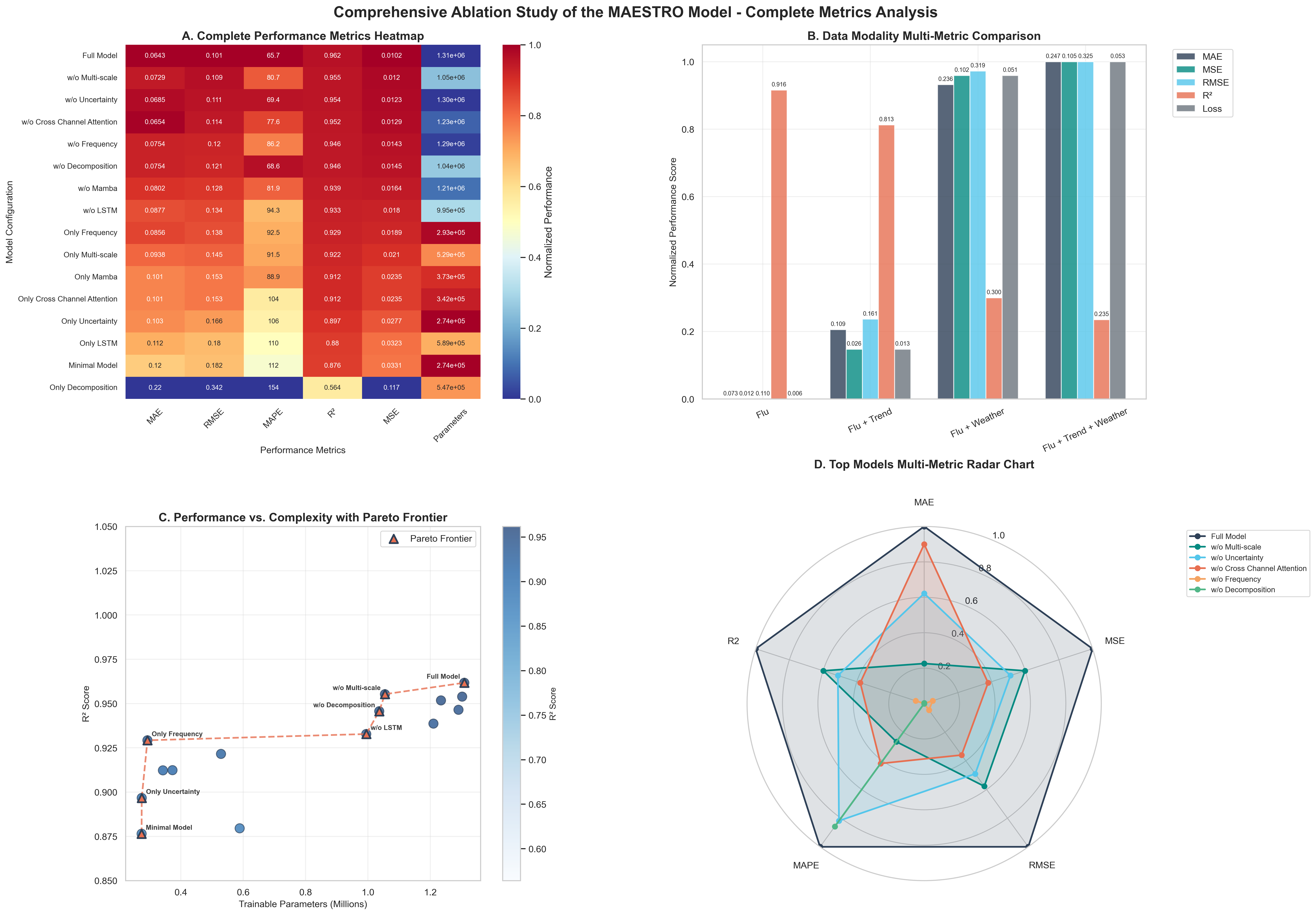}
    \caption{Comprehensive ablation study results: (A) Component impact analysis showing performance degradation when removing key components (w/o Mamba, w/o Cross-Channel Attention, etc.), (B) Data modality contribution analysis comparing different input combinations, (C) Performance-complexity trade-off analysis, and (D) Multi-metric radar chart comparing model variants across MAE, MSE, RMSE, and R\textsuperscript{2} metrics.}
    \label{fig:ablation-results}
\end{figure}

The comprehensive ablation study results presented in Figure~\ref{fig:ablation-results} provide detailed insights into MAESTRO's architecture design. Panel (A) demonstrates that removing the Mamba component causes the most significant performance degradation, followed by cross-channel attention and frequency domain components. Panel (B) shows that the full multi-modal approach (flu + trends + weather) achieves superior performance compared to single or dual modality combinations. The performance-complexity analysis in Panel (C), where complexity is measured by the number of trainable model parameters (in millions), reveals an optimal balance, while the radar chart in Panel (D) confirms MAESTRO's consistent superiority across multiple evaluation metrics.

\subsection{Error Analysis}
To further dissect the performance of MAESTRO against the baseline models, a detailed error analysis was conducted, as depicted in Figure \ref{fig:error_distribution}. The analysis includes a histogram of absolute errors, a box plot summarizing the error distributions, a cumulative distribution function (CDF) of the errors, and a statistical significance matrix comparing the models.

\begin{figure}[htbp]
  \centering
  \includegraphics[width=\textwidth]{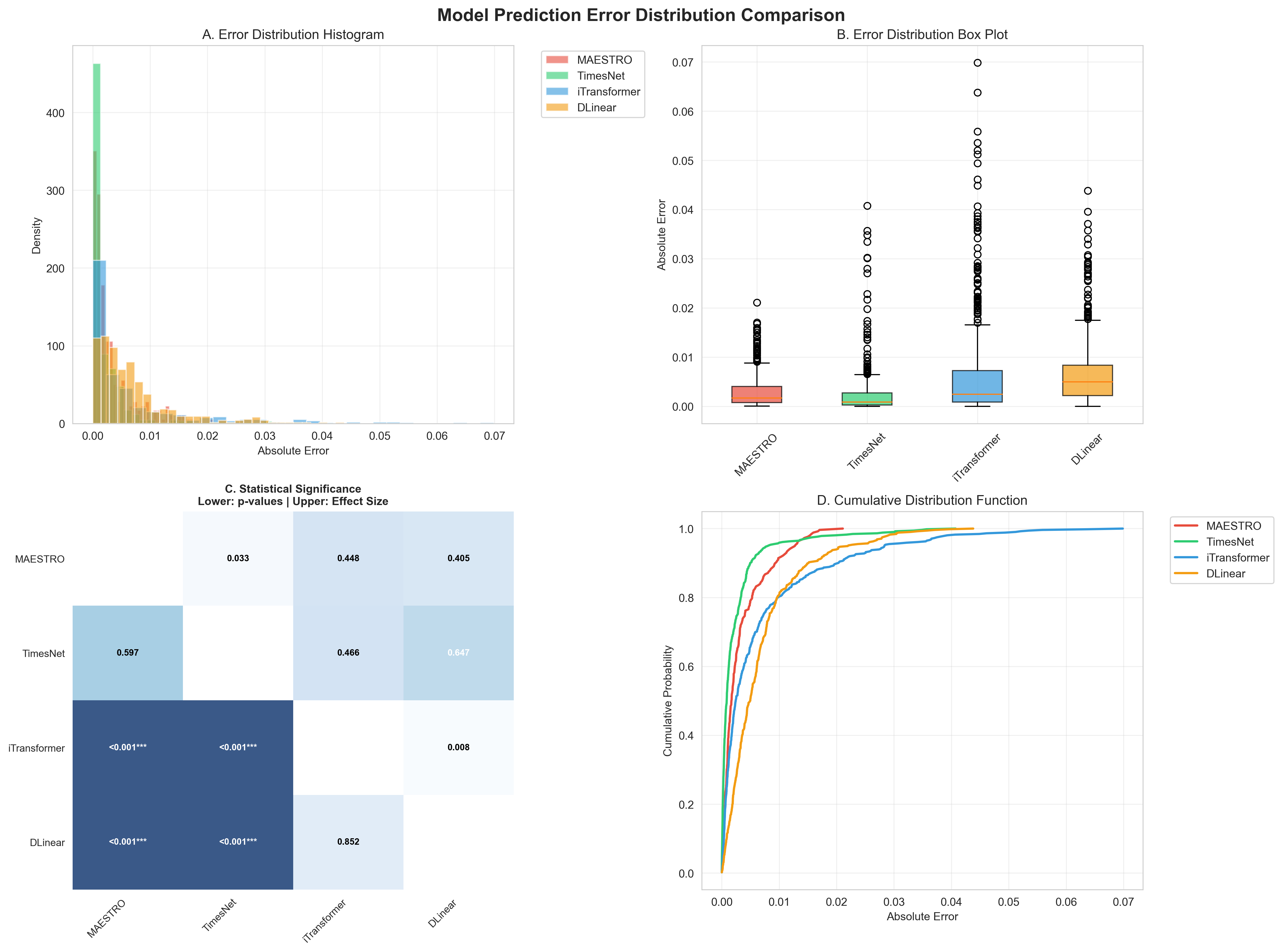}
  \caption{Model Prediction Error Distribution Comparison. This figure provides a comprehensive comparison of the absolute errors of MAESTRO and three baseline models (TimesNet, iTransformer, and DLinear). (A) A histogram of the error distributions. (B) A box plot for each model, highlighting the median, quartiles, and outliers. (C) The results of statistical significance tests (p-values and effect sizes between the models). (D) The cumulative distribution function (CDF) of the absolute errors, indicating the probability of the error being less than or equal to a certain value.}
  \label{fig:error_distribution}
\end{figure}

The error distribution histogram and the box plot reveal that MAESTRO and TimesNet generally have lower median errors and a tighter interquartile range compared to iTransformer and DLinear, suggesting more consistent and accurate predictions. The CDF plot further supports this, with MAESTRO's curve rising more steeply, indicating that a larger proportion of its predictions have very low errors.

The statistical significance matrix provides a quantitative comparison of the models. The lower triangle of the matrix shows the p-values from a Wilcoxon signed-rank test between pairs of models. The highly significant p-values ($<$0.001***) for the comparisons between MAESTRO and both iTransformer and DLinear confirm that MAESTRO's superior performance is statistically significant. While the difference between MAESTRO and TimesNet is not as statistically pronounced, the effect size (upper triangle) and the distributional plots suggest that MAESTRO still holds an edge in overall performance.

\section{Discussion and Limitations}
The results highlight the trade-offs inherent in multi-metric evaluation. While MAESTRO achieves a state-of-the-art R\textsuperscript{2} (0.956), indicating a superior overall fit to the data's variance, it is marginally outperformed by TimesNet on MAE and RMSE. This suggests that while the spectro-temporal fusion architecture excels at capturing the overall trend and seasonality, other models may be more specialized in minimizing point-wise absolute errors. The universally high MAPE values, as seen in Table~\ref{tab:sota}, are a known artifact of time series with near-zero values (common in off-seasons), which can disproportionately inflate the metric. For clarity, MAPE is reported calculated using the standard formula without any smoothing or adjustments for these near-zero values. It is thus contended that R\textsuperscript{2}, MAE, and RMSE are more reliable indicators for this forecasting task, with R\textsuperscript{2} underscoring the model's strength in explaining the underlying data structure.

The architectural gains of MAESTRO are primarily demonstrated empirically through extensive ablation studies (see Figure~\ref{fig:ablation-results}). A preliminary analysis of spectro-temporal factorization, efficient long-range modeling via SSM, and cross-modal fusion was provided. However, a formal theoretical analysis of why this specific combination of components is effective, particularly regarding its generalization performance under distribution shifts, remains a promising future research direction. The model's complexity, while justified by its performance, also warrants further investigation into more lightweight yet equally robust configurations.

This study is conducted on data from a single geography (Hong Kong), which limits the direct generalizability of the findings to other regions with different climates, healthcare systems, or population behaviors. Cross-region validation and domain adaptation are therefore critical next steps. Furthermore, while the COVID-19 period was intentionally excluded to ensure a stable training regime, the model's robustness to unforeseen structural breaks and regime shifts, such as pandemics, has not been fully evaluated and merits dedicated future study. This naturally leads to research on adaptive learning horizons and online model retraining strategies.

\section{Conclusion}
MAESTRO is introduced, a spectro-temporal multi-modal framework that fuses surveillance, search trends, and weather for influenza forecasting. The architecture couples decomposition, attention, state-space modeling, multi-scale convolutions, and frequency-domain learning. On Hong Kong data, it achieves a state-of-the-art R\textsuperscript{2} of 0.956 while demonstrating competitive performance on other metrics. Future work includes cross-region transfer, adaptive horizons, and theoretical analysis of spectro-temporal fusion.

\begin{acks}
The Hong Kong Department of Health is acknowledged for influenza surveillance, Google Trends for search data, and the Hong Kong Observatory for meteorological data. This work was supported in part by institutional funds from Macau University of Science and Technology. Colleagues are also thanked for helpful discussions.
\end{acks}

\section{Online Resources}
The source code for MAESTRO is publicly available in a GitHub repository: \url{https://github.com/Hylouis233/MAESTRO}.
\bibliographystyle{ACM-Reference-Format}
\bibliography{references}


\begin{thebibliography}{20}


\ifx \showCODEN    \undefined \def \showCODEN     #1{\unskip}     \fi
\ifx \showISBNx    \undefined \def \showISBNx     #1{\unskip}     \fi
\ifx \showISBNxiii \undefined \def \showISBNxiii  #1{\unskip}     \fi
\ifx \showISSN     \undefined \def \showISSN      #1{\unskip}     \fi
\ifx \showLCCN     \undefined \def \showLCCN      #1{\unskip}     \fi
\ifx \shownote     \undefined \def \shownote      #1{#1}          \fi
\ifx \showarticletitle \undefined \def \showarticletitle #1{#1}   \fi
\ifx \showURL      \undefined \def \showURL       {\relax}        \fi
\providecommand\bibfield[2]{#2}
\providecommand\bibinfo[2]{#2}
\providecommand\natexlab[1]{#1}
\providecommand\showeprint[2][]{arXiv:#2}

\bibitem[Box et~al\mbox{.}(2015)]%
        {box2015time}
\bibfield{author}{\bibinfo{person}{George E.~P. Box},
  \bibinfo{person}{Gwilym~M. Jenkins}, \bibinfo{person}{Gregory~C. Reinsel},
  {and} \bibinfo{person}{Greta~M. Ljung}.} \bibinfo{year}{2015}\natexlab{}.
\newblock \bibinfo{booktitle}{\emph{{Time Series Analysis: Forecasting and
  Control}} (\bibinfo{edition}{5th} ed.)}.
\newblock \bibinfo{publisher}{Wiley}, \bibinfo{address}{Hoboken, NJ}.
\newblock
\showISBNx{978-1-118-67502-1}


\bibitem[Cleveland et~al\mbox{.}(1990)]%
        {cleveland1990stl}
\bibfield{author}{\bibinfo{person}{Robert~B Cleveland},
  \bibinfo{person}{William~S Cleveland}, \bibinfo{person}{Jean~E McRae},
  \bibinfo{person}{Irma Terpenning}, {et~al\mbox{.}}}
  \bibinfo{year}{1990}\natexlab{}.
\newblock \showarticletitle{STL: A seasonal-trend decomposition}.
\newblock \bibinfo{journal}{\emph{J. off. Stat}} \bibinfo{volume}{6},
  \bibinfo{number}{1} (\bibinfo{year}{1990}), \bibinfo{pages}{3--73}.
\newblock


\bibitem[Ginsberg et~al\mbox{.}(2009)]%
        {ginsberg2009detecting}
\bibfield{author}{\bibinfo{person}{Jeremy Ginsberg},
  \bibinfo{person}{Matthew~H. Mohebbi}, \bibinfo{person}{Rajan~S. Patel},
  \bibinfo{person}{Lincoln Brammer}, \bibinfo{person}{Mark~S. Smolinski}, {and}
  \bibinfo{person}{Larry Brilliant}.} \bibinfo{year}{2009}\natexlab{}.
\newblock \showarticletitle{{Detecting influenza epidemics using search engine
  query data}}.
\newblock \bibinfo{journal}{\emph{Nature}} \bibinfo{volume}{457},
  \bibinfo{number}{7232} (\bibinfo{year}{2009}), \bibinfo{pages}{1012--1014}.
\newblock
\href{https://doi.org/10.1038/nature07634}{doi:\nolinkurl{10.1038/nature07634}}


\bibitem[Gu and Dao(2023)]%
        {gu2023mamba}
\bibfield{author}{\bibinfo{person}{Albert Gu} {and} \bibinfo{person}{Tri Dao}.}
  \bibinfo{year}{2023}\natexlab{}.
\newblock \showarticletitle{{Mamba: Linear-Time Sequence Modeling with
  Selective State Spaces}}. In \bibinfo{booktitle}{\emph{Thirty-seventh
  Conference on Neural Information Processing Systems}}.
\newblock
\urldef\tempurl%
\url{https://openreview.net/forum?id=AL1fq05o7H}
\showURL{%
\tempurl}


\bibitem[Gu et~al\mbox{.}(2022)]%
        {gu2022s4}
\bibfield{author}{\bibinfo{person}{Albert Gu}, \bibinfo{person}{Karan Goel},
  {and} \bibinfo{person}{Christopher R{\'{e}}}.}
  \bibinfo{year}{2022}\natexlab{}.
\newblock \showarticletitle{{Efficiently Modeling Long Sequences with
  Structured State Spaces}}. In \bibinfo{booktitle}{\emph{International
  Conference on Learning Representations, {ICLR} 2022}}.
  \bibinfo{publisher}{OpenReview.net}.
\newblock
\urldef\tempurl%
\url{https://openreview.net/forum?id=uYLFoz1vlAC}
\showURL{%
\tempurl}


\bibitem[Kingma and Ba(2015)]%
        {kingma2015adam}
\bibfield{author}{\bibinfo{person}{Diederik~P. Kingma} {and}
  \bibinfo{person}{Jimmy Ba}.} \bibinfo{year}{2015}\natexlab{}.
\newblock \showarticletitle{{Adam: A Method for Stochastic Optimization}}. In
  \bibinfo{booktitle}{\emph{Proceedings of the 3rd International Conference on
  Learning Representations, {ICLR} 2015}},
  \bibfield{editor}{\bibinfo{person}{Yoshua Bengio} {and} \bibinfo{person}{Yann
  LeCun}} (Eds.). \bibinfo{publisher}{OpenReview.net}.
\newblock
\urldef\tempurl%
\url{https://arxiv.org/abs/1412.6980}
\showURL{%
\tempurl}


\bibitem[Kuleshov et~al\mbox{.}(2018)]%
        {kuleshov2018accurate}
\bibfield{author}{\bibinfo{person}{Volodymyr Kuleshov}, \bibinfo{person}{Nathan
  Fenner}, {and} \bibinfo{person}{Stefano Ermon}.}
  \bibinfo{year}{2018}\natexlab{}.
\newblock \showarticletitle{Accurate uncertainties for deep learning using
  calibrated regression}. In \bibinfo{booktitle}{\emph{International conference
  on machine learning}}. PMLR, \bibinfo{pages}{2796--2804}.
\newblock


\bibitem[Li et~al\mbox{.}(2025a)]%
        {LiTIM2025AFD}
\bibfield{author}{\bibinfo{person}{Xuanfeng Li}, \bibinfo{person}{Haining He},
  \bibinfo{person}{Jiahong Huang}, \bibinfo{person}{Ruibin Liu},
  \bibinfo{person}{Tao Qian}, {and} \bibinfo{person}{Chitin Hon}.}
  \bibinfo{year}{2025}\natexlab{a}.
\newblock \showarticletitle{ASFM-AFD: Multimodal Fusion of AFD-Optimized LiDAR
  and Camera Data for Paper Defect Detection}.
\newblock \bibinfo{journal}{\emph{IEEE Transactions on Instrumentation and
  Measurement}}  \bibinfo{volume}{74} (\bibinfo{year}{2025}),
  \bibinfo{pages}{1--14}.
\newblock


\bibitem[Li et~al\mbox{.}(2025b)]%
        {LiTIM2025Spiro}
\bibfield{author}{\bibinfo{person}{Xuanfeng Li}, \bibinfo{person}{Yuxi Lin},
  \bibinfo{person}{Wei He}, \bibinfo{person}{Ruibin Liu},
  \bibinfo{person}{Arlindo~L. Oliveira}, \bibinfo{person}{Tao Qian},
  \bibinfo{person}{Jinping Zheng}, {and} \bibinfo{person}{Chitin Hon}.}
  \bibinfo{year}{2025}\natexlab{b}.
\newblock \showarticletitle{Enhancing the Interpretation of Spirometry: Joint
  Utilization of n-Order Adaptive Fourier Decomposition and Deep Learning
  Techniques}.
\newblock \bibinfo{journal}{\emph{IEEE Transactions on Instrumentation and
  Measurement}}  \bibinfo{volume}{74} (\bibinfo{year}{2025}),
  \bibinfo{pages}{1--14}.
\newblock


\bibitem[Li et~al\mbox{.}(2024)]%
        {LiJSEN2024YOLO}
\bibfield{author}{\bibinfo{person}{Xuanfeng Li}, \bibinfo{person}{Haibo Yan},
  \bibinfo{person}{Kai Cui}, \bibinfo{person}{Zhiwu Li},
  \bibinfo{person}{Ruibin Liu}, \bibinfo{person}{Guibin Lu},
  \bibinfo{person}{Kai~Chin Hsieh}, \bibinfo{person}{Xiaoshi Liu}, {and}
  \bibinfo{person}{Chitin Hon}.} \bibinfo{year}{2024}\natexlab{}.
\newblock \showarticletitle{A Novel Hybrid YOLO Approach for Precise Paper
  Defect Detection With a Dual-Layer Template and an Attention Mechanism}.
\newblock \bibinfo{journal}{\emph{IEEE Sensors Journal}} \bibinfo{volume}{24},
  \bibinfo{number}{7} (\bibinfo{year}{2024}), \bibinfo{pages}{11651--11669}.
\newblock


\bibitem[Li and Yu(2022)]%
        {LiERA2022Bird}
\bibfield{author}{\bibinfo{person}{Xuanfeng Li} {and} \bibinfo{person}{Jiajia
  Yu}.} \bibinfo{year}{2022}\natexlab{}.
\newblock \showarticletitle{Joint attention mechanism for the design of
  anti-bird collision accident detection system}.
\newblock \bibinfo{journal}{\emph{Electron. Res. Arch.}} \bibinfo{volume}{30},
  \bibinfo{number}{12} (\bibinfo{year}{2022}), \bibinfo{pages}{4401--4415}.
\newblock


\bibitem[Senanayake et~al\mbox{.}(2016)]%
        {senanayake2016predicting}
\bibfield{author}{\bibinfo{person}{Ransalu Senanayake}, \bibinfo{person}{Kasun
  Madhawa}, \bibinfo{person}{Hansa Kalutarage}, {and} \bibinfo{person}{Roshan
  Alwis}.} \bibinfo{year}{2016}\natexlab{}.
\newblock \showarticletitle{{Predicting Disease Trends using Internet Search
  Data: A Multi-Source Fusion Approach}}. In
  \bibinfo{booktitle}{\emph{Proceedings of the 22nd {ACM} {SIGKDD}
  International Conference on Knowledge Discovery and Data Mining}}.
  \bibinfo{publisher}{{ACM}}, \bibinfo{pages}{1845--1854}.
\newblock
\href{https://doi.org/10.1145/2939672.2939869}{doi:\nolinkurl{10.1145/2939672.2939869}}


\bibitem[Shaman et~al\mbox{.}(2010)]%
        {shaman2010absolute}
\bibfield{author}{\bibinfo{person}{Jeffrey Shaman},
  \bibinfo{person}{Virginia~E. Pitzer}, \bibinfo{person}{C{\'{e}}cile Viboud},
  \bibinfo{person}{Bryan~T. Grenfell}, {and} \bibinfo{person}{Marc Lipsitch}.}
  \bibinfo{year}{2010}\natexlab{}.
\newblock \showarticletitle{{Absolute humidity and the seasonal onset of
  influenza in the continental US}}.
\newblock \bibinfo{journal}{\emph{{PLoS Biology}}} \bibinfo{volume}{8},
  \bibinfo{number}{2} (\bibinfo{year}{2010}), \bibinfo{pages}{e1000316}.
\newblock
\href{https://doi.org/10.1371/journal.pbio.1000316}{doi:\nolinkurl{10.1371/journal.pbio.1000316}}


\bibitem[Vaswani et~al\mbox{.}(2017)]%
        {vaswani2017attention}
\bibfield{author}{\bibinfo{person}{Ashish Vaswani}, \bibinfo{person}{Noam
  Shazeer}, \bibinfo{person}{Niki Parmar}, \bibinfo{person}{Jakob Uszkoreit},
  \bibinfo{person}{Llion Jones}, \bibinfo{person}{Aidan~N. Gomez},
  \bibinfo{person}{{\L}ukasz Kaiser}, {and} \bibinfo{person}{Illia
  Polosukhin}.} \bibinfo{year}{2017}\natexlab{}.
\newblock \showarticletitle{{Attention Is All You Need}}. In
  \bibinfo{booktitle}{\emph{Advances in Neural Information Processing
  Systems}}, \bibfield{editor}{\bibinfo{person}{Isabelle Guyon},
  \bibinfo{person}{Ulrike~von Luxburg}, \bibinfo{person}{Samy Bengio},
  \bibinfo{person}{Hanna~M. Wallach}, \bibinfo{person}{Rob Fergus},
  \bibinfo{person}{S.~V.~N. Vishwanathan}, {and} \bibinfo{person}{Roman
  Garnett}} (Eds.), Vol.~\bibinfo{volume}{30}. \bibinfo{publisher}{Curran
  Associates, Inc.}, \bibinfo{pages}{5998--6008}.
\newblock
\urldef\tempurl%
\url{https://papers.nips.cc/paper/2017/hash/3f5ee243547dee91fbd053c1c4a845aa-Abstract.html}
\showURL{%
\tempurl}


\bibitem[Wu et~al\mbox{.}(2023)]%
        {wu2023timesnet}
\bibfield{author}{\bibinfo{person}{Haixu Wu}, \bibinfo{person}{Tengge Hu},
  \bibinfo{person}{Yong Liu}, \bibinfo{person}{Hang Zhou},
  \bibinfo{person}{Jianmin Wang}, {and} \bibinfo{person}{Mingsheng Long}.}
  \bibinfo{year}{2023}\natexlab{}.
\newblock \showarticletitle{{TimesNet: Temporal 2D-Variation Modeling for
  General Time Series Analysis}}. In \bibinfo{booktitle}{\emph{The Eleventh
  International Conference on Learning Representations, {ICLR} 2023}}.
  \bibinfo{publisher}{OpenReview.net}.
\newblock
\urldef\tempurl%
\url{https://openreview.net/forum?id=ju_Uqw384Oq}
\showURL{%
\tempurl}


\bibitem[Wu et~al\mbox{.}(2021)]%
        {wu2021autoformer}
\bibfield{author}{\bibinfo{person}{Haixu Wu}, \bibinfo{person}{Jiehui Xu},
  \bibinfo{person}{Jianmin Wang}, {and} \bibinfo{person}{Mingsheng Long}.}
  \bibinfo{year}{2021}\natexlab{}.
\newblock \showarticletitle{{Autoformer: Decomposition Transformers with
  Auto-Correlation for Long-Term Series Forecasting}}. In
  \bibinfo{booktitle}{\emph{Advances in Neural Information Processing
  Systems}}, \bibfield{editor}{\bibinfo{person}{Marc'Aurelio Ranzato},
  \bibinfo{person}{Alina Beygelzimer}, \bibinfo{person}{Yann~N. Dauphin},
  \bibinfo{person}{Percy~S. Liang}, {and} \bibinfo{person}{Jennifer~Wortman
  Vaughan}} (Eds.), Vol.~\bibinfo{volume}{34}. \bibinfo{publisher}{Curran
  Associates, Inc.}, \bibinfo{pages}{22419--22430}.
\newblock
\urldef\tempurl%
\url{https://proceedings.neurips.cc/paper/2021/hash/2DE5D16682C32EA5BB155FE3783434BE-Abstract.html}
\showURL{%
\tempurl}


\bibitem[Wu et~al\mbox{.}(2018)]%
        {wu2018deepflu}
\bibfield{author}{\bibinfo{person}{Shixiang Wu}, \bibinfo{person}{Yatao Wang},
  \bibinfo{person}{Cheng Long}, \bibinfo{person}{Zuo-Jun Wang}, {and}
  \bibinfo{person}{Cong Liu}.} \bibinfo{year}{2018}\natexlab{}.
\newblock \showarticletitle{{DeepFlu: A Deep Learning Framework for Flu
  Forecasting}}. In \bibinfo{booktitle}{\emph{Proceedings of the 24th {ACM}
  {SIGKDD} International Conference on Knowledge Discovery and Data Mining}}.
  \bibinfo{publisher}{{ACM}}, \bibinfo{pages}{2505--2514}.
\newblock
\href{https://doi.org/10.1145/3219819.3220084}{doi:\nolinkurl{10.1145/3219819.3220084}}


\bibitem[Zeng et~al\mbox{.}(2023)]%
        {zeng2023dlinear}
\bibfield{author}{\bibinfo{person}{Ailing Zeng}, \bibinfo{person}{Muxi Chen},
  \bibinfo{person}{Lei Zhang}, {and} \bibinfo{person}{Qiang Xu}.}
  \bibinfo{year}{2023}\natexlab{}.
\newblock \showarticletitle{{Are Transformers Effective for Time Series
  Forecasting?}}. In \bibinfo{booktitle}{\emph{Proceedings of the
  Thirty-Seventh {AAAI} Conference on Artificial Intelligence}},
  Vol.~\bibinfo{volume}{37}. \bibinfo{publisher}{{AAAI} Press},
  \bibinfo{pages}{11121--11128}.
\newblock
\href{https://doi.org/10.1609/aaai.v37i9.26317}{doi:\nolinkurl{10.1609/aaai.v37i9.26317}}


\bibitem[Zhou et~al\mbox{.}(2021)]%
        {zhou2021informer}
\bibfield{author}{\bibinfo{person}{Haoyi Zhou}, \bibinfo{person}{Shanghang
  Zhang}, \bibinfo{person}{Jieqi Peng}, \bibinfo{person}{Shuai Zhang},
  \bibinfo{person}{Jianmin Li}, \bibinfo{person}{Hui Xiong}, {and}
  \bibinfo{person}{Wancai Zhang}.} \bibinfo{year}{2021}\natexlab{}.
\newblock \showarticletitle{{Informer: Beyond Efficient Transformer for Long
  Sequence Time-Series Forecasting}}. In \bibinfo{booktitle}{\emph{Proceedings
  of the Thirty-Fifth {AAAI} Conference on Artificial Intelligence}},
  Vol.~\bibinfo{volume}{35}. \bibinfo{publisher}{{AAAI} Press},
  \bibinfo{pages}{11106--11115}.
\newblock
\href{https://doi.org/10.1609/aaai.v35i12.17325}{doi:\nolinkurl{10.1609/aaai.v35i12.17325}}


\bibitem[Zhou et~al\mbox{.}(2022)]%
        {zhou2022fedformer}
\bibfield{author}{\bibinfo{person}{Tian Zhou}, \bibinfo{person}{Ziqing Ma},
  \bibinfo{person}{Qingsong Wen}, \bibinfo{person}{Xue Wang},
  \bibinfo{person}{Liang Sun}, {and} \bibinfo{person}{Rong Jin}.}
  \bibinfo{year}{2022}\natexlab{}.
\newblock \showarticletitle{{FEDformer: Frequency Enhanced Decomposed
  Transformer for Long-term Series Forecasting}}. In
  \bibinfo{booktitle}{\emph{Proceedings of the 39th International Conference on
  Machine Learning}} \emph{(\bibinfo{series}{Proceedings of Machine Learning
  Research}, Vol.~\bibinfo{volume}{162})}. \bibinfo{publisher}{{PMLR}},
  \bibinfo{pages}{27268--27286}.
\newblock
\urldef\tempurl%
\url{https://proceedings.mlr.press/v162/zhou22g.html}
\showURL{%
\tempurl}


\end{thebibliography}

\appendix

\end{document}